\documentclass{article}

\usepackage[main, final]{neurips_2026}

\usepackage[utf8]{inputenc} %
\usepackage[T1]{fontenc}    %
\usepackage{hyperref}       %
\usepackage{url}            %
\usepackage{booktabs}       %
\usepackage{amsfonts}       %
\usepackage{nicefrac}       %
\usepackage{microtype}      %
\usepackage{xcolor}         %

\usepackage{graphicx}
\usepackage{subcaption}
\usepackage[textsize=tiny]{todonotes}
\usepackage{soul}
\usepackage{enumitem}
\usepackage{dsfont}
\usepackage{amsmath}
\usepackage{amssymb}
\usepackage{mathtools}
\usepackage{float}
\usepackage[linesnumbered,ruled,vlined]{algorithm2e}

\usepackage{bbm}
\usepackage{makecell}
\usepackage{wasysym}
\usepackage{wrapfig}
\usepackage{multirow}
\usepackage{tabularray}

\newcommand{\X}{\mathcal{X}}
\newcommand{\Y}{\mathcal{Y}}
\newcommand{\Pd}{\mathbb{P}}
\newcommand{\Qd}{\mathbb{Q}}

\newcommand{\Rdx}{\mathbb{R}^{d_x}}
\newcommand{\Rdy}{\mathbb{R}^{d_y}}
\newcommand{\R}{\mathbb{R}}

\newcommand{\x}{\textbf{x}}
\newcommand{\y}{\textbf{y}}

\newcommand{\op}[1]{\operatorname{#1}}
\newcommand{\argmin}{\operatornamewithlimits{argmin}}
\newcommand{\argmax}{\operatornamewithlimits{argmax}}
\newcommand{\<}{\langle}
\renewcommand{\>}{\rangle}
\newcommand*{\eqdef}{\stackrel{\text{\normalfont def}}{=}}

\newtheorem{theorem}{Theorem}[section]

\newtheorem{lemma}[theorem]{Lemma}

\newcommand{\bbP}{\mathbb{P}}
\newcommand{\bbQ}{\mathbb{Q}}
\newcommand{\cP}{\mathcal{P}}

\title{Uncovering Challenges of Solving\\ the Continuous Gromov-Wasserstein Problem}

\author{%
    Xavier Aramayo-Carrasco \\
    Applied AI Institute\\
    Moscow, Russia \\
    \texttt{xavier.aramayo2@gmail.com} \\
    \And
    Maksim Nekrashevich\thanks{Work was done while at Applied AI Institute} \\
    Nebius AI\\
    Amsterdam, Netherlands \\
    \And
    Petr Mokrov \\
    Applied AI Institute\\
    Moscow, Russia \\
    \And
    Evgeny Burnaev \\
    Applied AI Institute, AXXX\\
    Moscow, Russia \\
    \And
    Alexander Korotin \\
    Applied AI Institute, AXXX\\
    Moscow, Russia \\
}

\begin{document}

\maketitle

\vspace{-4mm}

\begin{abstract}
\vspace{-3mm}
Recently, the Gromov-Wasserstein Optimal Transport (GWOT) problem has attracted the special attention of the ML community. In this problem, given two distributions supported on two (possibly different) spaces, one has to find the most isometric map between them. In the discrete variant of GWOT, the task is to learn an assignment between given discrete sets of points. In the more advanced continuous formulation, one aims at recovering a parametric mapping between unknown continuous distributions based on i.i.d. samples derived from them. The clear geometrical intuition behind the GWOT makes it a natural choice for several practical use cases, giving rise to a number of proposed solvers. Some of them claim to solve the continuous version of the problem. At the same time, GWOT is notoriously hard, both theoretically and numerically. Moreover, all existing continuous GWOT solvers still heavily rely on discrete techniques. Natural questions arise: to what extent do existing methods unravel the GWOT problem, what difficulties do they encounter, and under which conditions they are successful? Our benchmark paper is an attempt to answer these questions. We specifically focus on the continuous GWOT as the most interesting and debatable setup. We crash-test existing continuous GWOT approaches on different scenarios, carefully record and analyze the obtained results, and identify issues. Our findings experimentally testify that the scientific community is still missing a reliable continuous GWOT solver, which necessitates further research efforts. As a first step in this direction, we propose a new continuous GWOT solver specifically designed for the inner-product setting, which does not rely on discrete techniques and addresses some of the limitations of existing approaches. Code available at: \url{https://github.com/Ark-130994/GWBench.git}.
\end{abstract}

\vspace{-5mm}

\addtocontents{toc}{\protect\setcounter{tocdepth}{0}}
\section{Introduction}
\label{sec:introduction}
\vspace{-2mm}

Optimal Transport (OT) is a powerful framework that is widely used in machine learning \citep{montesuma2023recent}. A popular application of OT is the domain adaptation of various modalities, including images \citep{courty2016optimal,luo2018wgan,redko2019optimal}, music transcription \citep{flamary2016optimal}, color transfer \citep{frigo2015optimal}, alignment of embedding spaces \citep{chen2020graph,aboagye2022quantized}. Other applications include generative models \citep{salimans2018improving,arjovsky2017wasserstein}, unpaired imgage-to-image transfer \citep{korotin2023neural,rout2022generative}, etc.

\begin{figure}[!ht]
\vspace{-3mm}
\centering
    \begin{subfigure}[t]{0.46\textwidth}
        \centering
        \includegraphics[width=0.8\linewidth]{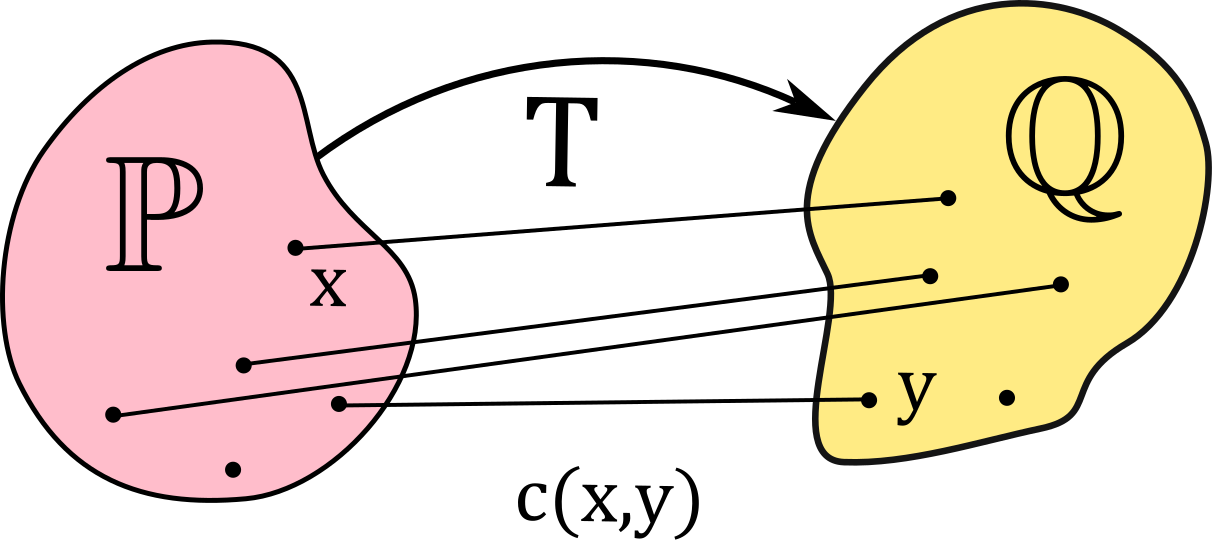}
        \caption{ Monge's OT between distributions $\mathbb{P}$ and $\mathbb{Q}$ with \protect\linebreak \textit{inter-domain} cost function $c(x,y)$.}
        \label{fig:OT-OLD}
    \end{subfigure}
    \hfill %
    \begin{subfigure}[t]{0.46\textwidth}
        \centering
        \includegraphics[width=0.8\linewidth]{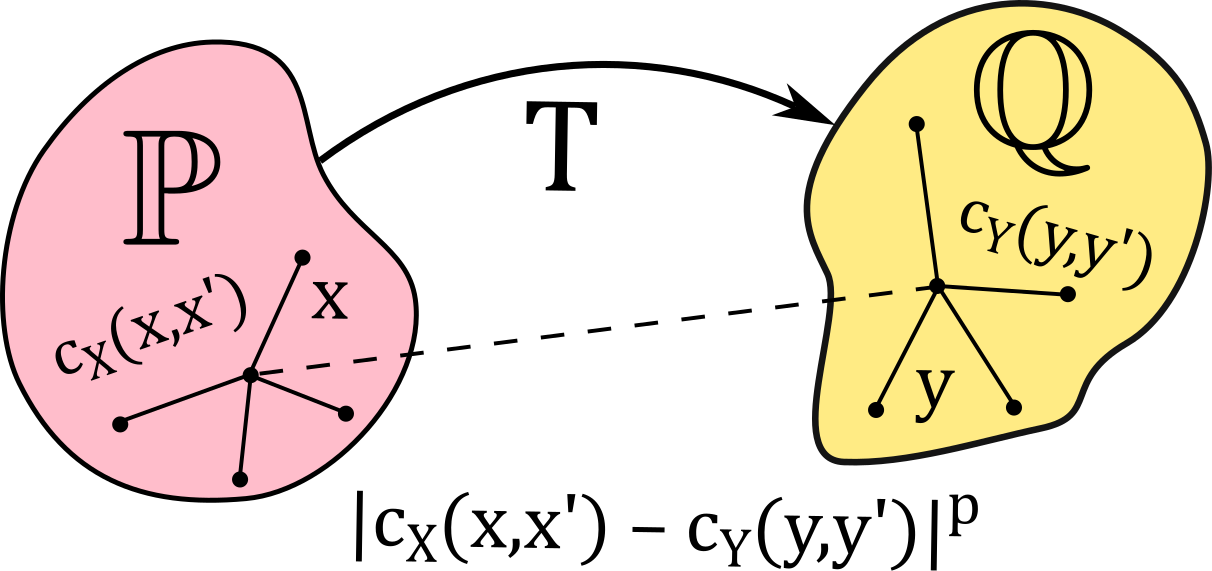}
        \caption{Monge's GW between distributions $\mathbb{P}$ and $\mathbb{Q}$ \protect\linebreak with \textit{intra-domain} costs $c_{\mathcal{X}}(x,x')$ and $c_{\mathcal{Y}}(y,y')$.}
        \label{fig:GWOT-OLD}
    \end{subfigure}
    \caption{A schematic visualization of the OT problems and GW problems (Monge's form).}
    \vspace{-1mm}
\end{figure}

In the conventional OT problem (Figure \ref{fig:OT-OLD}), one needs to find a map between two data distributions that minimizes a certain ``effort'' expressed in the form of an \emph{inter-domain} transport cost function. This cost function shows how hard it is to move a point in the source space to a given point in the target space. Thus, in order for the resulting map to possess certain useful properties, one has to incorporate them into the cost function. Unfortunately, this is not always a straightforward task, especially when the data distributions are supported in different spaces.

\vspace{-1.5mm}
A popular way to address the above-mentioned issue is to consider the Gromov-Wasserstein (GW) modification \citep{memoli2007use, memoli2011gromov, peyre2016gromov} of the OT problem (Figure \ref{fig:GWOT-OLD}). Here one assumes that both the source and target spaces are equipped with a structure, e.g., with a metric, and one aims to find a transport map that maximally preserves this structure, i.e., the most isometric map. This clear geometrical intuition behind GW makes it natural in various applications: unsupervised data alignment \citep{alvarez2018gromov,aboagye2022quantized}, single-cell data processing \citep{scetbon2022linear, klein2023generative, sebbouh2024structured, pmlr-v238-demetci24a, ryu2024crossmodality}, 2D and 3D shape analysis \citep{beier2022linear, memoli2009spectral}, graph data analysis \citep{xu2019scalable, xu2021learning, vincent-cuaz2022semirelaxed, chowdhury2021generalized, vincent2021online, mazelet2025unsupervisedlearningoptimaltransport} and dimensionality reduction \citep{vanassel2024distributionalreductionunifyingdimensionality}.

\vspace{-1mm}
\textbf{Discrete/continuous GW}. The GWOT problem is about learning some specific translation that operates with source and target distributions. In practice, these distributions are typically given by empirical datasets. This leads to two possible ways of paving a GWOT map. In the \textbf{discrete} scenario, the learned translation is just a point(s)-to-point(s) assignment (transport matrix). In turn, the \textbf{continuous} GW is about learning a parametric mapping between the underlining (continuous) distributions. In this case, the datasets are treated as i.i.d. samples derived from them.

\vspace{-1mm}
While existing computational approaches for the GW problem show considerable empirical success, the problem itself is highly non-trivial from different perspectives.

\vspace{-2mm}
\begin{itemize}[leftmargin=*]
    \item \textbf{Theory.} Finding the most isometric map between probability spaces based just on the inner properties of these spaces may be poorly defined, e.g., the desired transform may be non-unique. This happens where the source (target) space permits some non-trivial isometries that preserve the corresponding source (target) distribution. A simple yet expressive example is the Gaussian case \cite{delon2022gromov}. Intuitively, non-uniqueness may affect the stability of a GWOT solver.
    
    \vspace{-1mm}
    \item \textbf{Computations/algorithms.} It is known that the discrete GW yields a non-convex quadratic optimization problem \citep{vayer2020contribution}, which is computationally challenging. To partially alleviate the difficulty, one typical approach is to consider entropic regularization \citep{peyre2016gromov, alvarez2018gromov, scetbon2022linear, wang2023neural}. Fortunately, the regularized problem resorts to a sequence of tractable Sinkhorn OT assignments. However, the convergence of the procedure may not hold, see \citep[Remark 3]{peyre2016gromov}. In addition, discrete GWOT scales poorly with the number of input (source or target) samples, making some problem setups unmanageable by such solvers. While there are techniques to reduce the computational burden w.r.t. data size \citep{scetbon2022linear}, they come at the cost of additional restrictions and assumptions.

    \vspace{-1mm}
    \item \textbf{Methodology.} Majority of existing continuous GW methods are based on discrete GW techniques and inherit their computational challenges. Moreover, the transition from the discrete to the continuous setup may be questionable from a statistical point of view \citep{zhang2024gromov}.
\end{itemize}

\vspace{-2mm}
Having said that, one naturally wonders: how do the current continuous GWOT methods manage to overcome these problems and show good practical results? What are the ``bad cases'' under which the aforementioned difficulties become critical and the solvers struggle? How to fight with these ``bad cases''? In our paper, we shed light on these GWOT methods' ambiguities, specifically focusing on the continuous setup for the inner product and cosine distance cost functionals. \underline{\textbf{Our contributions:}}

\vspace{-2mm}
\begin{itemize}[leftmargin=*]
    \item We conduct a deep analysis of existing papers and reveal that the majority of works primarily consider datasets with some specific statistical relations between source and target samples. Formally speaking, such setups disobey the standard i.i.d. assumption on the data and may lead to ambiguous conclusions on the solvers' capabilities.%

    \vspace{-1mm}
    \item By following the findings from the previous point, we evaluate the performance of existing continuous inner product and cosine distance based GWOT solvers in more statistically fair and practically realistic \textit{statistically unrelated} data setups. Our experiments witness that the level of statistical relatedness indeed highly influences the quality of the learned GWOT maps. %
    
    \vspace{-1mm}
    \item Motivated by these findings, we propose a novel continuous neural GW solver for the inner-product setting that does not rely on statistical relatedness between source and target training samples. Our method is not based on discrete GW and it may be learned on arbitrarily large datasets. Since it can be trained on a large number of samples, the method exhibits reasonably good performance even in the challenging \textit{statistically unrelated} setup. On the other hand, the method is $\min$-$\max$-$\min$ adversarial, which impacts stability and requires plenty of data for training.

\end{itemize}

\vspace{-3mm}
Overall, our findings reveal that existing GWOT solvers can be effective for certain applications when statistically related data is available—a condition typically impractical for many real-world GW tasks. However, these methods fail to learn meaningful transport plans when this statistical relation is absent. To address this limitation, we propose a solver capable of handling fully statistically unrelated setups, enabling the GWOT approach to be practically useful in more realistic and challenging scenarios.

\vspace{-1mm}
\textbf{Notations.}
Throughout the paper, $\Rdx$ and $\Rdy$ are the source and target data spaces, respectively. The set of Borel probability distributions on $\Rdx$ is $\cP(\Rdx)$. The dot product of vectors $\x, \x' \in \Rdx$ is $\<\x, \x'\>_{d_x}$. For a measurable map $T \colon \Rdx \to \Rdy$, we denote the corresponding \emph{push-forward} operator by $T_\sharp$. 
For $\bbP \in \cP(\Rdx)$ and $\bbQ \in \cP(\Rdy)$, we denote the set of all couplings between them by $\Pi(\Pd, \Qd)$, i.e., distributions $\pi$ on $\Rdx \times \Rdy$ with the corresponding marginals equal to $\Pd$ and $\Qd$. %

\vspace{-2mm}
\section{Background}
\label{sec:background}
\vspace{-1mm}

In this section, we first explain the conventional OT setup \citep{villani2008optimal, santambrogio, gozlan2017kantorovich, backhoff2019existence} and then introduce the GWOT formulation \citep{memoli2011gromov, peyre2016gromov}. Finally, we clarify the practical learning setup for these problems.

\vspace{-1mm}
\subsection{Optimal Transport (OT) problem}
\label{sec:conventionalOT}

\vspace{-1mm}
Given two probability distributions $\Pd \in \cP(\Rdx)$, $\Qd \in \cP(\Rdy)$ and a cost function $c\colon\Rdx\times\Rdy\rightarrow\R$, the OT problem is defined as follows:

\vspace{-2mm}
\begin{equation}
    \text{OT}_{c}(\mathbb{P},\mathbb{Q})\eqdef\inf_{T_{\sharp} \bbP = \bbQ}\int_{\mathbb{R}^{d_x}}c(\textbf{x},T(\textbf{x}))d\mathbb{P}(\textbf{x}).
    \label{eq:monge-ot}
\end{equation}
\vspace{-1mm}

This is known as Monge's formulation of the OT problem. Intuitively, it can be understood as finding an optimal transport map $T^* : \Rdx \rightarrow \Rdy$ that transforms $\bbP$ to $\bbQ$ and minimizes the total transportation expenses w.r.t. cost $c$, see Figure \ref{fig:OT-OLD}. There have been developed a lot of methods for solving OT \eqref{eq:monge-ot} in its discrete \citep{cuturi2013sinkhorn, peyre2019computational} and continuous \citep{makkuva2020optimal, daniels2021score, korotin2023neural, choi2023generative, fan2023neural, uscidda2023monge, gushchin2024entropic, mokrov2024energyguided, asadulaev2024neural} variants.

\vspace{-1mm}
The cost function $c$ in \eqref{eq:monge-ot} is commonly the squared Euclidean distance, making problem \eqref{eq:monge-ot} exclusively defined for spaces of equal dimension. Handling incomparable spaces ($d_x \neq d_y$) requires a manually defined inter-domain cost, which is a nontrivial task \citep{wang2025quadraticformoptimaltransport}.

\vspace{-1mm}
\subsection{Gromov-Wasserstein OT (GWOT) problem}
\label{sec:GWOT}
\vspace{-1mm}

The GWOT problem is an extension of the optimal transport problem that aims to compare and transport probability distributions supported on different spaces. This problem is particularly useful when the underlying spaces do not align directly, but we still want to measure and align their intrinsic geometric structures. In what follows, we introduce the discrete and continuous variants of GWOT.

\vspace{-1mm}
\textbf{Discrete Gromov-Wasserstein formulation.}
Let $N_x$ and $N_y$ be the number of training samples in the source and target domains, respectively. Let $\textbf{C}^x \in \mathbb{R}^{N_x \times N_x}$ and $\textbf{C}^y \in \mathbb{R}^{N_y \times N_y}$ be the corresponding source and target intra-domain cost matrices. These matrices measure the pairwise distance or similarity between the samples for a given function, i.e., cosine similarity, Euclidean distance, inner product, etc. The discrete GWOT problem is:

\vspace{-2mm}
\begin{equation}
    \textbf{T}^* \eqdef \argmin_{\textbf{T} \in \mathcal{C}_{N_x, N_y}} \,\sum_{i,j,k,l} \vert \textbf{C}^x_{i,k} - \textbf{C}^y_{j,l} \vert^p\textbf{T}_{i,j}\textbf{T}_{k,l},
    \label{eq:gwot-problem-discrete}
\end{equation}

\vspace{-2mm}

here ${\mathcal{C}_{N_x, N_y}\!\!\eqdef\!\!\{ \textbf{T} \!\in \!\mathbb{R}^{N_x\!\times\!N_y}_+ \big\vert \textbf{T}^T \mathds{1}\!=\!\frac{1}{N_y}\mathds{1} ; \textbf{T} \mathds{1}\!=\!\frac{1}{N_x}\mathds{1} \}}$ is the set of coupling matrices between source and target spaces; $\mathds{1} \!= \![1, \dots, 1]^T\!\!\in \!\R^N, N \!\in\! \{N_x, N_y\}$.
The loss function $\vert\! \cdot \!- \!\cdot \!\vert^p$ in \eqref{eq:gwot-problem-discrete} accounts for the misfit between the similarity matrices, a typical choice for the degree factor is $p = 2$. %
We refer to \citep{peyre2016gromov, memoli2011gromov, chowdhury2019gromov, titouan2019sliced} for details.

\textbf{Continuous Gromov-Wasserstein formulation.} Let $\mathbb{P} \in \cP(\Rdx)$, $\mathbb{Q} \in \cP(\Rdy)$ be two distributions. Let $c_{\mathcal{X}}:\Rdx \times \Rdx \rightarrow\mathbb{R}$ and  $c_{\mathcal{Y}}:\Rdy \times \Rdy \rightarrow\mathbb{R}$ be two intra-domain cost functions for the source ($\mathbb{R}^{d_x}$) and target ($\mathbb{R}^{d_y}$) domains, respectively. The Monge's GWOT problem is defined as:

\vspace{-2mm}
\begin{equation}
    \text{GWOT}^p_p(\mathbb{P}, \mathbb{Q}) \eqdef \inf_{T_{\sharp} \bbP = \bbQ} \!\!\!\!\!\!\!\int\limits_{\Rdx\times\Rdx}\!\!\!\!\!\!\!\! \left|c_{\X}(\x, \x')\!-\!c_{\Y}(T(\x), T(\x'))\right|^p\!\!d\Pd(\x)d\Pd(\x').\label{eq:gwot-problem}
\end{equation}

\vspace{-1mm}

Theoretical results on the existence and regularity of \eqref{eq:gwot-problem} under certain cases could be found in \citep{dumont2024existence, memoli2024comparison, memoli2022distance}.
An illustration of problem \eqref{eq:gwot-problem} can be found in Figure \ref{fig:GWOT-OLD}. In this continuous setup, the objective is to find an optimal transport map $T^* : \Rdx \rightarrow \Rdy$ that allows to transform (align) the source distribution to the target distribution. While in \eqref{eq:monge-ot} we search for a map that sends $\Pd$ to $\Qd$ minimizing the total transport cost, \eqref{eq:gwot-problem} aims to find the most isometric map w.r.t. the costs $c_\X$ and $c_\Y$, i.e., the map that maximally preserves the pairwise intra-domain costs. The commonly studied case \citep{vayer2020contribution,sebbouh2024structured} is $p=2$ with the Euclidean distance $c(\cdot,\cdot)=\|\cdot-\cdot\|^2$ or inner product $c(\cdot,\cdot)=\langle\cdot,\cdot\rangle$ as intra-domain cost functions. In what follows, we will use \textbf{innerGW} to denote the latter case.
 
\vspace{-1mm}
\subsection{Practical Learning Setup}\label{subsec:training-setup}
\vspace{-1mm}

In practical scenarios, the source and target distributions $\mathbb{P}$ and $\mathbb{Q}$ are typically accessible by empirical samples (datasets) $X = \{ \x_i\}_{i = 1}^{N_x} \sim \bbP$ and $Y = \{ \y_i \}_{i = 1}^{N_y} \sim \bbQ$. Under the \textbf{discrete} GWOT formulation, these samples are directly used to compute intra-domain cost matrices $\textbf{C}^x$, $\textbf{C}^y$. These matrices are then fed to optimization problem \eqref{eq:gwot-problem-discrete}. Having been solved, problem \eqref{eq:gwot-problem-discrete} yields a coupling matrix $\textbf{T}^*$ which defines the GWOT correspondence between $X$ and $Y$. Importantly, discrete GWOT operates with discrete empirical measures $\hat{\bbP} \eqdef \sum_{i = 1}^{N_x} \frac{1}{N_x} \delta(\x_i)$, $\hat{\bbQ} \eqdef \sum_{i = 1}^{N_y} \frac{1}{N_y} \delta(\y_i)$ rather than original ones.

\vspace{-1mm}
In turn, under the \textbf{continuous} setup, the aim is to recover some parametric map $T^* : \R^{d_x} \rightarrow \R^{d_y}$ between the original source and target distributions $\bbP$ and $\bbQ$.  
In most practical scenarios, the latter is preferable, as it naturally allows \textit{out-of-sample} estimation, i.e., provides GWOT mapping for new (unseen) samples $\x \sim \bbP$. \underline{In our paper, we deal with the continuous setup}.

\vspace{-2mm}
\section{Continuous Gromov-Wasserstein solvers}
\label{sec:related_work}
\vspace{-1mm}

Here we review recent advances in solving the continuous GWOT formulation. Most of these GWOT solvers are only discrete or adapted to \textit{emulate a continuous behavior} by implementing some specific out-of-sample estimation method on top of the results of some discrete solver. The initial approach to solve the GWOT problem in discrete case (\wasyparagraph \ref{sec:GWOT}) was introduced in \citep{memoli2011gromov, peyre2016gromov}. Below we only detail the methods which specifically aim to solve the continuous formulation and provide out-of-sample estimation. A \underline{sanity-check} toy setup $\R^3\!\rightarrow\!\R^2$ is available in Appendix \ref{ap:toy3Dto2D}.

\vspace{-1mm}
\textbf{StructuredGW \citep{sebbouh2024structured}}.
It introduces an iterative algorithm for solving the discrete entropy-regularized inner product case of the GW problem with $p=2$, using the reformulation from \citep[maxOT]{vayer2020contribution}. At each step, the coupling matrix $\mathbf{T}$ is updated via Sinkhorn iterations, alongside an auxiliary rotation matrix is updated by various methods depending on the regularization. For out-of-sample estimation, the method uses entropic maps \citep{pooladian2024entropic, dumont2024existence}, using a learned dual potential to perform the mapping.

\vspace{-1mm}
\textbf{FlowGW \citep{klein2023generative}.} The method consists in fitting a discrete GW solver, e.g., \citep{peyre2016gromov}, to obtain a coupling matrix $\textbf{T}$. This coupling determines the best way to match source samples to target ones. Weighted pairs of source and target samples are constructed using the distribution described by the coupling matrix. Then these samples are used to train a Conditional Flow Matching (CFM) model \citep{lipman2023flow} with the noise outsourcing technique from \citep{kallenberg1997foundations}. The inference process is done by solving the ODE given by the CFM model.

\vspace{-1mm}
\textbf{AlignGW  \citep{alvarez2018gromov}}. This work proposes a discrete solver for the alignment of word embeddings. The authors use  Sinkhorn iterations to compute the updates on the coupling matrix instead of a linear search implemented in the Python Optimal Transport library \citep{flamary2021pot} %
that is inspired by \citep{titouan2019optimal, peyre2016gromov}. This change significantly improves the stability of the solver. In spite of being a totally discrete method, we consider it due to its performance in challenging tasks. In order to allow out-of-sample estimations, we train a Multi-Layer Perceptron (MLP) model on the argmax based matrix from the learned coupling matrix $\mathbf{T}$. %

\vspace{-1mm}
\textbf{RegGW \citep{uscidda2024disentangledrepresentationlearninggeometry}}. This work extends the Monge Gap Regularizer \citep{uscidda2023monge} to the Gromov-Wasserstein setting by parameterizing the transport push-forward function $T$ with a neural network. The model is trained using a regularized loss that combines a \textit{fitting} term—chosen as the Sinkhorn divergence \citep{uscidda2024disentangledrepresentationlearninggeometry}—to align the mapped points with the target distribution, and the Gromov-Monge gap regularizer, defined as the difference between the distortion of $T$ (i.e., its GW cost) and the true discrete GW solution. Unlike other methods such as \citep{sotiropoulou2024stronglyisomorphicneuraloptimal}, this approach does not require access to an intermediate reference distribution, making it more practical for our experimental setting.

\vspace{-1mm}
\textbf{CycleGW \citep{Zhang2021CycleCP}}.
The authors propose minimizing the Unbalanced Bidirectional Gromov–Monge Divergence (UBGMD) to recover two push-forward maps: \(f\) pushing $\bbP$ toward $\bbQ$, and $g$ doing the reverse. This is conceptually related to the Unbalanced Gromov-Wasserstein divergence \cite{sejourne2020unbalanced}. Their approach involves minimizing a Generalized Maximum Mean Discrepancy (GMMD), computed using MMD with Gaussian kernels. In our setup, $f$ corresponds to our mapping function $T$. A related method from \cite{2021arXiv210914090H} also incorporates an MMD term, but due to its similarity and lack of publicly available code, we focus on \cite{Zhang2021CycleCP}. 

\vspace{-2mm}

\section{Limitations of existing methods}

\vspace{-1mm}
As mentioned before, the majority of existing Gromov-Wasserstein approaches rely on discrete GWOT formulation, see \wasyparagraph \ref{sec:GWOT}. Solving \eqref{eq:gwot-problem-discrete} becomes increasingly computationally expensive as the training sample sizes $N_x$ and $N_y$ grow. This renders some datasets to hardly be manageable by discrete solvers. This limitation gives rise to additional issues, which we describe in this section.

\vspace{-3mm}
\subsection{Pitfalls of practical data setup}\label{subsec:correlated_setups}

\vspace{-1mm}
To begin with, we introduce the notion of \textbf{statistical (un)relatedness} of data which undergoes Gromov-Wasserstein alignment.
To fit a GWOT solver, a practitioner typically has two training datasets, $X = \{ \x_i\}_{i = 1}^{N_x}$ and $Y = \{ \y_i \}_{i = 1}^{N_y}$. They are sampled from the source ($\bbP$) and target ($\bbQ$) distributions, see our training setup, \wasyparagraph \ref{subsec:training-setup}. In what follows, without loss of generality, we will assume $N_x = N_y \eqdef N$. The natural statistical assumption on samples $X$ and $Y$ is that they are mutually independent.
We define this data setup as \textit{statistically unrelated}. Simultaneously, we introduce an alternative setup under which the source and target datasets $X$ and $Y$ turn out to be connected by some specific statistical relationships. Let us assume that samples $X$ and $Y$ are generated according to a joint distribution $\pi \in \Pi(\Pd, \Qd)$, $\pi\ne \Pd\otimes \Qd$. Practically, we expect that samples from this distribution are meaningfully dependent; i.e., $\pi$ is “significantly” different from the independent coupling $\Pd\otimes \Qd$ due to the structural similarities present in the respective intra-domain geometries. In particular, coupling $\pi$ may even set a one-to-one correspondence between the domains. Then, we call the training samples $X$ and $Y$ to be \textit{statistically related} if they are obtained as follows:

\vspace{-2mm}
\begin{enumerate}[leftmargin=*]
    \item First, we jointly sample $X$ and $\widetilde{Y} = \{\widetilde{\y}_{i} \}_{i = 1}^{N} \subset \Rdy$ from coupling $\pi$, i.e.: $X \times \widetilde{Y} \sim \pi$.

    \vspace{-2mm}
    \item Secondly, we apply \textit{an unknown permutation} $\sigma$ of indices to $\widetilde{Y}$ yielding $Y$, i.e.: $Y = \sigma \circ \widetilde{Y}$. 
\end{enumerate}
\vspace{-2mm}

We found that the majority of the experimental test cases, on which the existing GWOT solvers are validated, frequently follow exactly the \textit{statistically related} data setup. For instance, in the problem of learning cross-lingual word embedding correspondence \citep{alvarez2018gromov, grave2019unsupervised}, the underlining coupling $\pi$ could be understood as the distribution of dictionary pairs. The other example is the bone marrow dataset \citep{luecken2021a, klein2023generative}. In this case, the source and target samples are generated using two different methods to profile gene expressions on the same donors.
Interestingly, the \textit{statistically related} data setup suits well the discrete GWOT formulation, because the optimization problem in this case boils down to the search for the permutation $\sigma$ that spawned the target dataset $Y$. This leads to the natural \textbf{hypothesis} that for \textit{statistically unrelated} training datasets $X$ and $Y$, the performance of existing GWOT solvers may be poor. To test this hypothesis, we propose the experimental framework described below.

\begin{wrapfigure}{r}{0.43\textwidth}
\vspace{-2mm}
    \begin{subfigure}{0.4\textwidth}
    \centering
    \includegraphics[scale=0.9]{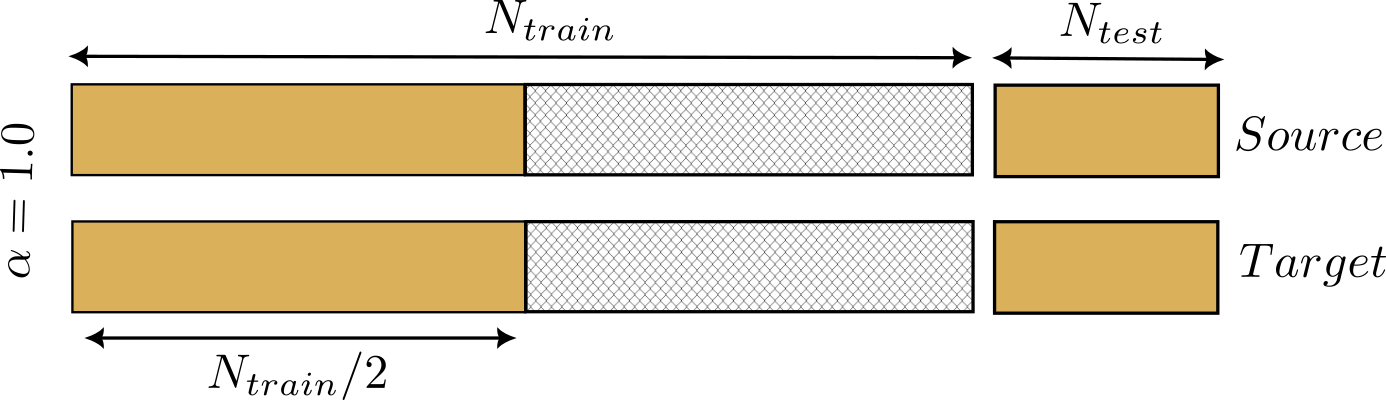} 
    \vspace*{-5mm}\caption{\centering\small Totally statistically related setup. }
    \label{fig:totally}
    \end{subfigure}
    \newline
    \begin{subfigure}{0.4\textwidth}
    \centering
    \includegraphics[scale=0.9]{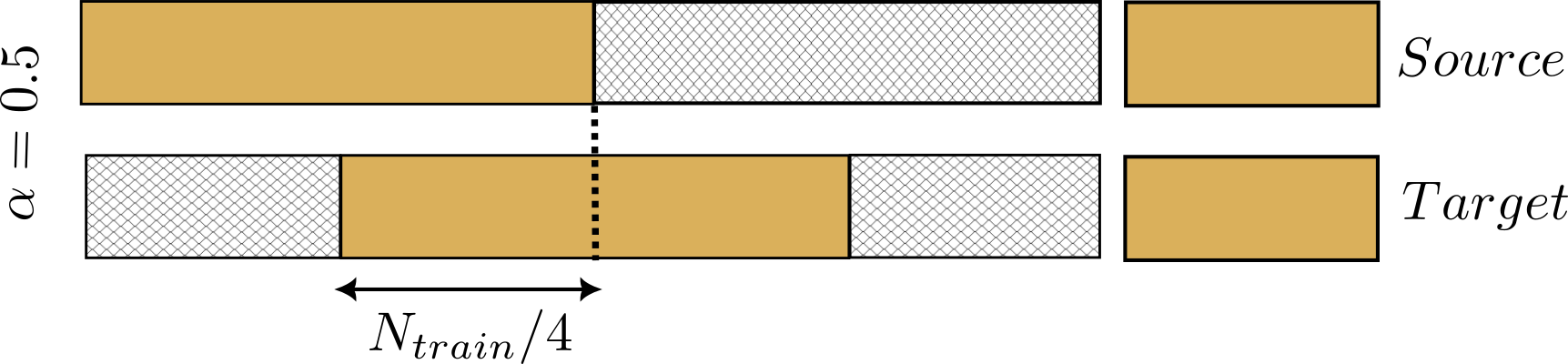}
    \vspace*{-5mm}\caption{\centering Partially statistically related setup.}
    \label{fig:partially}
    \end{subfigure}
    \newline
    \begin{subfigure}{0.4\textwidth}
    \centering
    \includegraphics[scale=0.9]{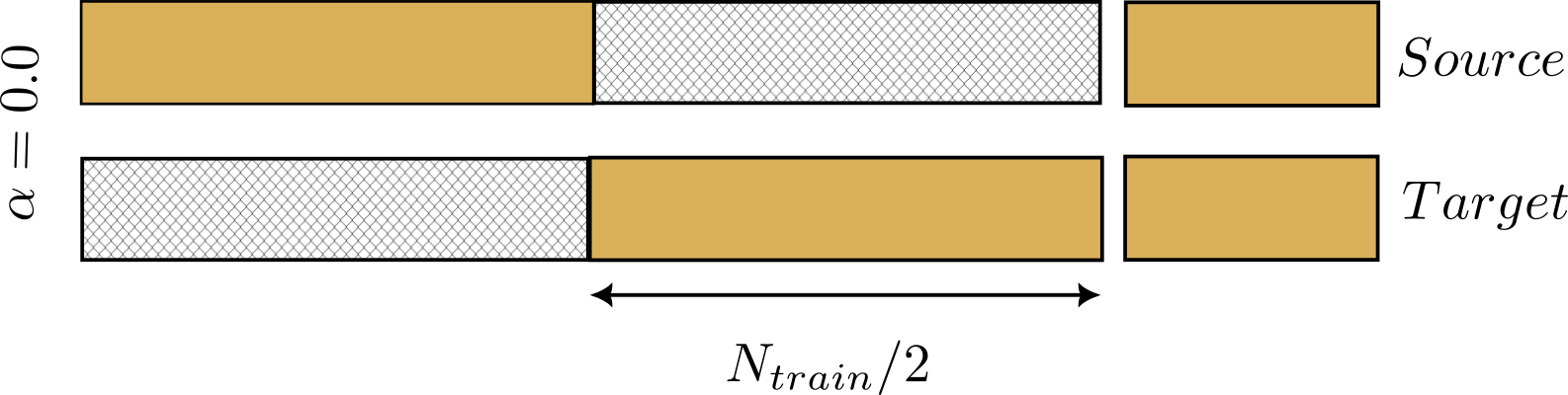}
    \vspace*{-5.5mm}\caption{Statistically unrelated setup.}
    \label{fig:unpaired}
    \end{subfigure}
    \vspace*{-2mm}\caption{ Data splitting and statistical (un)relatedness.}
    \label{fig:pairness}
\end{wrapfigure}

\textbf{Modeling statistical (un)relatedness in practice.}
Let $X=\{\x_1,\dots,\x_N\}\sim \Pd$ and ${Y=\{\y_1,\dots,\y_N\}\sim \Qd}$ be the source and target datasets. We assume that $X$ and $Y$ are initially totally paired, i.e., every $i$-th vector in the source set $X$ is the pair of the $i$-th vector in the target set, $Y$. Also, we suppose that the pairing is reasonable, i.e., dictated by the nature of the data on hand. For example, if $X$ and $Y$ are word embeddings, then $\x_i$ and $\y_i$ correspond to the same word. We propose a way to create training data with different levels of statistical relatedness given by $\alpha \in[0,1]$. First, random batches of $N$ paired source-target samples are split into train and test sets ($N = N_{\text{train}} + N_{\text{test}}$). 

\vspace{-1mm}
For training, source data takes the first $N_{\text{train}}/2$ samples. For target, we select $N_{\text{train}}/2$ consecutive samples from index $\lceil (1-\alpha) N_{\text{train}}/2 \rceil$ and then randomly shuffle them. $\alpha=1$ indicates a \textit{totally statistically related} setup, $0<\alpha<1$ \textit{partially related}, and $\alpha=0$ \textit{statistically unrelated} (Figure~\ref{fig:pairness}).

It is important to emphasize that our benchmark construction requires datasets with \textbf{(a)} known ground-truth pairs, \textbf{(b)} alignable structures (ground-truth pairs are optimal pairs, i.e., they yield minimal distortion) and \textbf{(c)} large amount of samples to train continuous models.

\vspace{-0.8mm}
To conclude the subsection, we want to remark that both \textit{statistically related} and \textit{statistically unrelated} setups are practically important. The \textit{statistically unrelated} setup naturally appears when aligning single-cell multi-omics data \citep{demetci2022scot}, where the goal is to match different single-cell assays. In this setting, the assays are statistically unrelated, since applying multiple assays to the same single cell is typically infeasible.

\vspace{-2mm}

\subsection{Benchmarking GWOT solvers on statistically (un)related data} %
\label{subsec:benchmark}

\vspace{-2mm}
In order to check how continuous GW solvers perform under statistically unrelated, partially and totally statistically related setups, we utilize \textbf{word} and \textbf{voice} embeddings. 

     \textbf{Word Embeddings.} We consider two text corpora named Twitter\footnote{https://radimrehurek.com/gensim/downloader.html} and MUSE\footnote{https://github.com/facebookresearch/MUSE} (Multilingual Unsupervised and Supervised Embeddings) \citep{DBLP:journals/corr/abs-1710-04087} bilingual vocabularies. We embed the Twitter vocabulary using the GloVe (Global Vectors for Word Representation) algorithm \citep{pennington2014glove} and the MUSE vocabulary using BPEmb (Byte-Pair) embeddings \citep{heinzerling2018bpemb}, respectively. We refer to these as Twitter-GloVe and MUSE-BPEmb.
     
     \textbf{Voice Embeddings.} We consider the VoxCeleb2 dataset\footnote{https://www.robots.ox.ac.uk/~vgg/data/voxceleb/vox2.html}. To obtain the embeddings we use a pretrained WavLM model \citep{Chen2021WavLM} and the embedding process from \citep{gengembre24_interspeech}. We name it VoxCeleb2-WNTA, following the naming convention of the original work \footnote{https://huggingface.co/Orange/Speaker-wavLM-pro}.

In this subsection we test \textbf{three} baseline solvers introduced in \wasyparagraph\ref{sec:related_work}: StructuredGW, AlignGW, and FlowGW. To perform the training and evaluation of the baseline solvers, we introduce two types of target data pools: the \textit{whole} data space and the \textit{reference} space. The \textit{whole} space consists on a big subset of samples ($400\text{K}$ for Twitter-GloVe, $90\text{K}$ for MUSE-BPEmb from English dictionary and $600\text{K}$ for VoxCeleb2-WNTA). The \textit{reference} space refers to the $N_{\text{train}}+N_{\text{test}}$ samples in Figure \ref{fig:pairness}.

We fit each solver for different values of $\alpha$ (levels of statistical relatedness), ranging from 0.0 to 1.0 in increments of 0.1. For each $\alpha$, we perform \textbf{ten} repetitions with different random seeds following the process described in \wasyparagraph\ \ref{subsec:correlated_setups}. We use $N_{\text{train}} = 6$K, i.e., each run uses source and target datasets containing $N_{\text{train}}/2 = 3$K samples; $N_{\text{test}} = 2048$. The only reason for this relatively small training size is the computational complexity of the three solvers. The discrete optimization procedures underlying them cannot scale to significantly larger $N_{\text{train}}$. The training of the RegGW and CycleGW solvers from \S\ref{sec:related_work} fails on small datasets, so we defer them to \S\ref{sec:exp_neuralgw}, where larger datasets are considered. The results are shown in Figure \ref{fig:discrete_main_text}. We refer to Appendix \ref{ap:experiments} for \ul{additional experimental setups}.

In the main text, we show results for different source and target embedding dimensions. For Twitter-GloVe, we use the setup $100\! \rightarrow\! 50$ (Figure~\ref{fig:discrete-twitter-glove-100-50}), and for MUSE-BPEmb, the same $100 \!\rightarrow \!50$ setup applies (Figure~\ref{fig:discrete-MUSE-BP-100-50}), where $100$ is the source dimension and $50$ the target. For VoxCeleb2-WNTA, we consider the $128 \rightarrow 64$ setup (see Figure \ref{fig:discrete-voxceleba-wnta-128-64}). \ul{Additional configurations} are presented in Appendix~\ref{ap:experiments}.

\begin{figure*}[h]

    \begin{subfigure}{1\textwidth}
    \centering
    \includegraphics[width=\linewidth]{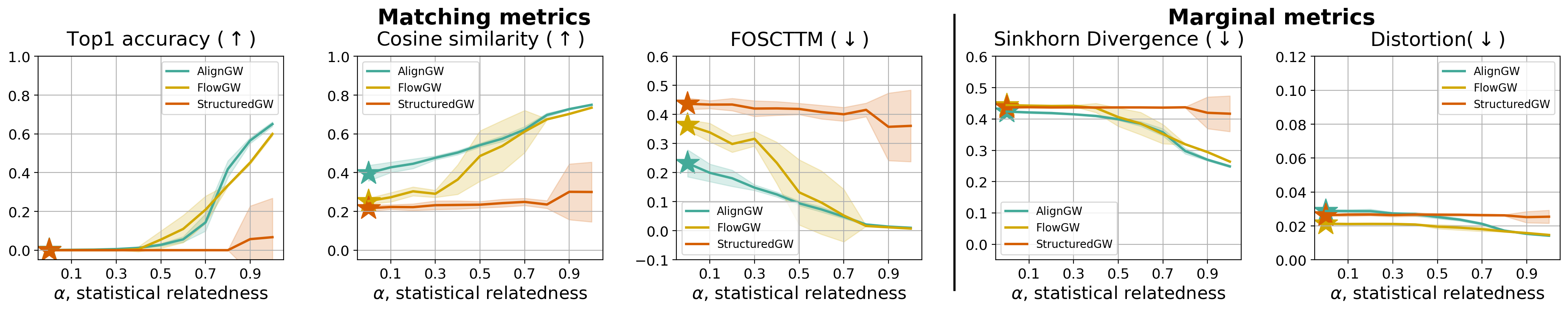} 
    \vspace{-5mm}
    \caption{\small Twitter-GloVe dataset ($\text{Source:}100\rightarrow \text{Target:}50$)}
    \label{fig:discrete-twitter-glove-100-50}
    \end{subfigure}

    \begin{subfigure}{1\textwidth}
    \centering
    \includegraphics[width=\linewidth]{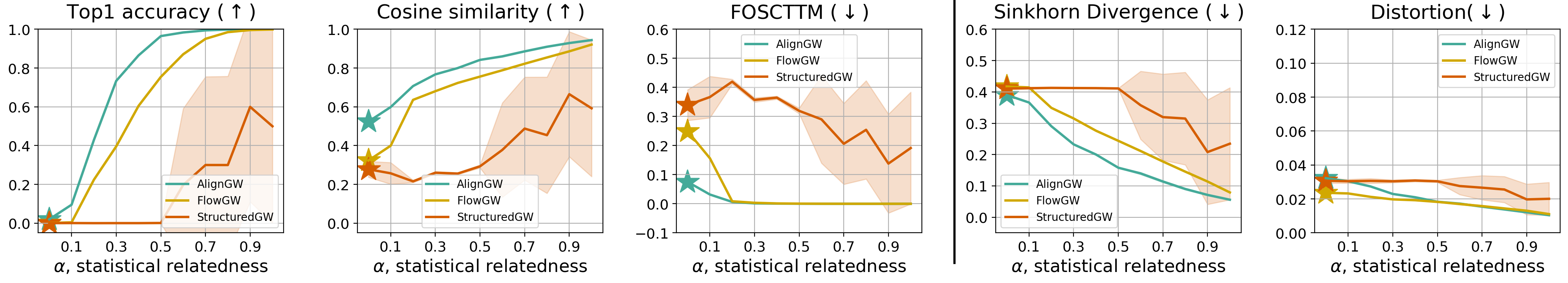} 
    \vspace{-5mm}\caption{\small MUSE-BPEmb (English) dataset ($\text{Source:}100\rightarrow \text{Target:}50$)}
    \label{fig:discrete-MUSE-BP-100-50}
    \end{subfigure}

    \begin{subfigure}{1\textwidth}
    \centering
    \includegraphics[width=\linewidth]{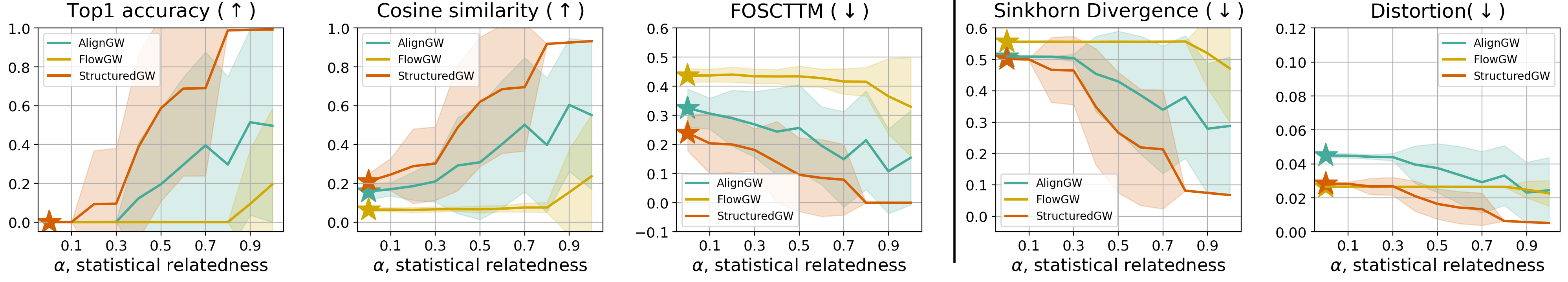} 
    \vspace{-5mm}\caption{\small VoxCeleb2-WNTA dataset ($\text{Source:}128\rightarrow \text{Target:}64$)}
    \label{fig:discrete-voxceleba-wnta-128-64}
    \end{subfigure}

    \vspace{-2mm}
    \caption{\small Performance of the baseline GWOT solvers for \textbf{Twitter-GloVe} (\ref{fig:discrete-twitter-glove-100-50}), \textbf{MUSE-BPEmb (English)} (\ref{fig:discrete-MUSE-BP-100-50}) and \textbf{VoxCeleb2-WNTA} (\ref{fig:discrete-voxceleba-wnta-128-64}) at different levels of statistical relatedness $\alpha$. Solvers were trained with $N_{\text{train}}/2 = 3000$ samples from spaces of $400\text{K}$, $90\text{K}$ and $600\text{K}$, respectively. The plot shows results on a 2048-sample test subset. Accuracy was computed on the \emph{reference} space of $N_{\text{train}} \!+ \!N_{\text{test}} \!=\! 8048$ samples.}
    \label{fig:discrete_main_text}
    
\end{figure*}

\vspace{-1mm}
\textbf{Evaluation.}  In the main text we report matching metrics (Top-$1$ accuracy $\uparrow$, cosine similarity $\uparrow$, and FOSCTTM $\downarrow$) and marginal metrics (Distortion $\downarrow$ and Sinkhorn divergence $\downarrow$). Additional marginal metrics and a justification for the chosen metrics are given in Appendices~\ref{ap:marginal_metrics} and~\ref{ap:metrics_theory}, respectively. All metrics are computed on unseen test data. Top-$1$ accuracy is computed using the \textit{reference} space; results for the \textit{whole} data space are in Appendix~\ref{ap:main_experiments_whole_reference}. 

\vspace{-1mm}
\textbf{Conclusions.} The results indicate that almost all the baseline solvers perform well in totally statistically related scenarios, even when evaluated on unseen data. StructuredGW presents a less robust behavior. This demonstrates their ability to learn and capture the inner structures when the data is highly statistically related. However, their performance drops significantly as the value of $\alpha$ decreases. \ul{This finding highlights a key limitation:} current GWOT methods are restricted to setups where source and target datasets are statistically related. This is particularly significant, as it shows that these methods fail to learn meaningful structure-preserving mappings when optimal source–target pairs are absent from the training data. Consequently, the applicability of GW solvers is limited, since their goal is to recover relational structure rather than individual optimal pairings.%

\vspace{0.5mm}
\textbf{An attempt to solve the pitfalls of existing GWOT solvers.}  We conjecture that the observed behavior is mainly due to the small sizes of training sets dictated by the discrete nature of the solvers. Indeed, relatively small samples hardly could fully express the intrinsic geometry of the data, which complicates the faithful GW alignment of the domains, see Appendix \ref{app-discrete-gw-issue} for the broader discussion on the issue.
Therefore, one possible way to improve how GW solvers perform on statistically unrelated data is to create GW solvers capable of handling large datasets. A natural direction is to construct solvers parameterized by neural networks, which scale well to large datasets via gradient-based optimization. In the following section (\S \ref{sec:gwlargescale}) we propose a new continuous solver (\S \ref{subsec:neuralgw}) which does not rely on discrete techniques. Therefore, it can better capture the inner geometry and structure of the data without the strict need of training on statistically related data.

\vspace{-2mm}
\section{GWOT solvers at large scale}
\label{sec:gwlargescale}
\vspace{-1mm}

In this section, we start by introducing NeuralGW, a novel scalable method to solve the continuous GWOT problem (\wasyparagraph \ref{subsec:neuralgw}). Then we proceed to the practical performance of NeuralGW and the baselines in large-scale GloVe benchmark (\wasyparagraph \ref{sec:exp_neuralgw}).

\vspace{-1mm}
\subsection{Neural Gromov-Wasserstein Solver}
\label{subsec:neuralgw}

\vspace{-1mm}
In this subsection, we conduct the theoretical and algorithmic derivation of our proposed approach. In what follows, we restrict to the case $d_x \geq d_y$; source ($\bbP$) and target ($\bbQ$) distributions are absolutely continuous and supported on some compact subsets $\mathcal{X} \subset \Rdx$, $\mathcal{Y} \subset \Rdy$ respectively.
Our method is developed for innerGW, i.e., problem \eqref{eq:gwot-problem} with $c_\X\!=\!\langle\cdot,\cdot\rangle_{d_x}$,  $c_\Y\!=\!\langle\cdot,\cdot\rangle_{d_y}$, $p=2$:

\begin{equation}
    \op{innerGW}^2_2(\Pd, \Qd) \eqdef \min_{T_{\sharp}\Pd=\Qd}\!\!\!\!\!\int\limits_{\Rdx \!\times\! \Rdx}\!\!\!\!\!\!\big|\<\x, \x'\>_{d_x} \!-\! \<T(\x), T(\x')\>_{d_y}\big|^2 \!\!d\Pd(\x)d\Pd(\x').\label{eq:innerGW}
\end{equation}

\vspace{-1mm}
Note that the existence of minimizer for \eqref{eq:innerGW} is due to \citep[Theorem 3.2]{dumont2024existence}. We base our method on the theoretical insights about GW from \citep{vayer2020contribution}. According to \citep[Theorem 4.2.1]{vayer2020contribution}, when $\int \|\x\|_2^4d\Pd(\x) < +\infty$, $\int \|\y\|_2^4d\Qd(\y) < +\infty$, problem \eqref{eq:innerGW} is equivalent to

\vspace{-1mm}
\begin{equation}\label{eq:maxot}
  \op{innerGW}^2_2(\Pd, \Qd) = \op{Const}(\Pd, \Qd) -  \!\!\!\max_{\pi \in \Pi(\Pd, \Qd)}\max_{P \in F_{d_x, d_y}} \int \<\textbf{P}\x, \y\>_{d_y}d\pi(\x, \y),
\end{equation}

\vspace{-1mm}

where $F_{d_x, d_y} \!\!\eqdef\! \{\textbf{P} \in \R^{d_y \times d_x}\:|\:\|\textbf{P}\|_{\mathcal{F}} = \min(\sqrt{d_x}, \sqrt{d_y})\}$ are the matrices of fixed Frobenius norm. Note that \eqref{eq:maxot} admits a solution $\pi^* \in \Pi(\Pd, \Qd)$, $P^* \in F_{d_x, d_y}$ \citep[Lemmas 6.2.7; 4.2.2]{vayer2020contribution}.

\vspace{-1mm}
Our following lemma reformulates the innerGW problem as a minimax optimization problem. This  reformulation is inspired by well-celebrated dual OT solvers such as \citep{korotin2021neural,fan2023neural,korotin2023neural}.

\vspace{-1mm}
\begin{lemma}[InnerGW as a minimax optimization] It holds that \eqref{eq:maxot} is equivalent to

\vspace{-1 mm}
\begin{equation}
\op{innerGW}_2^2(\mathbb{P}, \mathbb{Q})=\op{Const}{(\mathbb{P},\mathbb{Q})}+ \min_{P\in F_{d_x,d_y}}\,\max_{f:\Rdy\rightarrow\mathbb{R}} \,\,\min_{T \!\colon\! \Rdx \to \Rdy} \mathcal{L}(\textbf{P}, f, T);
\label{eq:neuralgw}
\end{equation} 

\vspace{-1mm}
\begin{equation*}
\mathcal{L}(\textbf{P}, f, T) \eqdef \int_{\Rdy}f(\textbf{y})d\mathbb{Q}(\textbf{y})-\int_{\Rdx} [\langle \textbf{P}\textbf{x},T(\textbf{x})\rangle_{d_y} +f(T(\textbf{x}))]d\mathbb{P}(\textbf{x}).
\end{equation*} 
\label{lem:innergw}
\end{lemma}

\vspace{-1mm}

The follwoign theorem shows that, for any existing optimal GW map $T^*$, one can select a
compatible $P^*$ so that $T^*$ minimizes \eqref{eq:neuralgw} for every optimal dual potential $f^*$.

\begin{theorem}[Optimal maps solve the minimax problem]\label{thm:GW}
Assume that there exists at least one optimal GW map $T^*$. For any such map, let $\textbf{P}^*$ be a matrix satisfying

\vspace{-1mm}
$$
\textbf{P}^* \in \argmax_{\textbf{P}\in F_{d_x,d_y}} \int_{\Rdx} \langle \textbf{P}\textbf{x},T^*(\textbf{x})\rangle\,d\mu(\textbf{x}).
$$

\vspace{-1mm}
Then $\textbf{P}^*$ solves the outer minimization in \eqref{eq:neuralgw}. For any potential $f^*$ that solves

\vspace{-1mm}
$$
f^* \in \argmax_f \min_{T \colon \Rdx \to \Rdy} \mathcal{L}(\textbf{P}^*, f, T),
$$

\vspace{-1mm}
we have:
\vspace{-1mm}
\begin{equation}
    T^* \in \argmin_{T \colon \Rdx \to \Rdy} \mathcal{L}(\textbf{P}^*, f^*, T).
\end{equation}
\end{theorem}

\vspace{-1mm}
To optimize \eqref{eq:neuralgw} we follow the best practices from the field of continuous OT \citep{korotin2021continuous,fan2023neural,korotin2023kernel,korotin2023neural,asadulaev2024neural,choi2023generative,gushchin2024entropic} and simply parameterize $f$ and $T$ with neural networks. In turn, $\textbf{P}$ is a learnable matrix of fixed Frobenius norm. We use the alternating stochastic gradient descent/ascent/descent method to train their parameters. The learning algorithm is detailed in Appendix \ref{ap:code}. We call the method NeuralGW. The hard min-max-min optimization inherent to the method (Equation \eqref{eq:neuralgw}) can lead to stability challenges. We alleviate these by adopting a specific \underline{initialization strategy} for the parameterized networks $T$, $f$, and $\textbf{P}$ (see Appendix~\ref{ap:neuralgw_init}). As a sanity check, we run our proposed approach on toy 3D$\rightarrow$2D setup, see Figure \ref{fig:toyneuralv2} in Appendix \ref{ap:toy3Dto2D}. Note that in comparison to other continuous GWOT approaches, our method \textit{does not rely on discrete OT} in any form. In particular, the training process assumes access to just random samples from $\mathbb{P},\mathbb{Q}$; it does not use/need any pairing between them.

\vspace{-1mm}
\subsection{Practical performance of NeuralGW and baselines at large scale}
\label{sec:exp_neuralgw}
\vspace{-1mm}

We start by introducing our \textbf{large-scale} Twitter-GloVe, MUSE-BPEmb and VoxCeleb2-WNTA setups. We consider the same data split process as in \wasyparagraph \ref{subsec:benchmark}, but we respectively take \mbox{$N_{train}\in\{398\text{K},88\text{K},598\text{K}\}$} samples instead of the $6$K used in \wasyparagraph \ref{subsec:benchmark} for the baseline solvers; $N_{test}=2048$ is left the same. Each repetition consists in training the models with the same data but using different initialization parameters for, e.g., the neural networks. We computed accuracy on the \textit{whole} target space. In the large-scale setup, the \textit{whole} space coincides with the \textit{reference} space.

\vspace{-1mm}

\begin{figure*}[!ht]
    \begin{subfigure}{1\textwidth}
    \centering
    \includegraphics[width=0.95\linewidth]{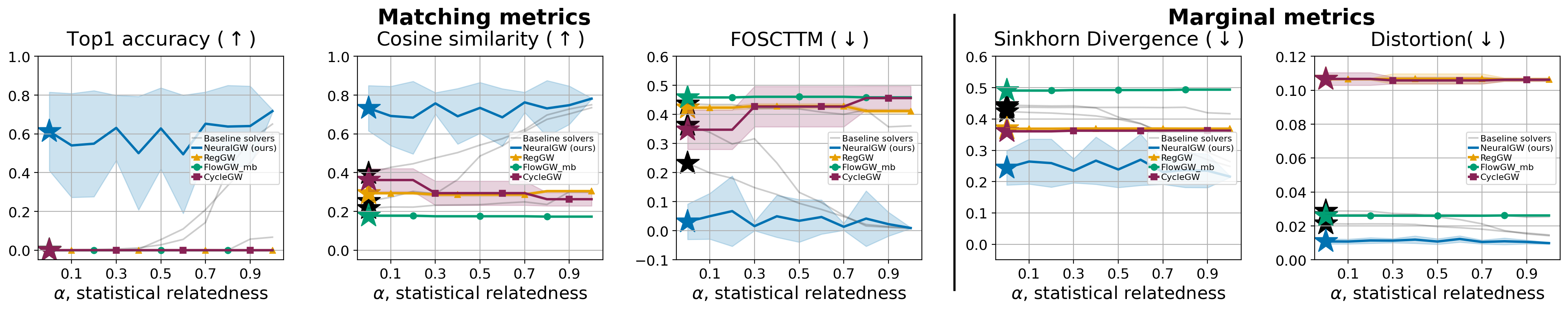} 
    \caption{\small Twitter-GloVe dataset ($\text{Source:}100\rightarrow \text{Target:}50$)}
    \label{fig:continuous_twitter_glove_100_50}
    \end{subfigure}
    
    \begin{subfigure}{1\textwidth}
    \centering
    \includegraphics[width=0.95\linewidth]{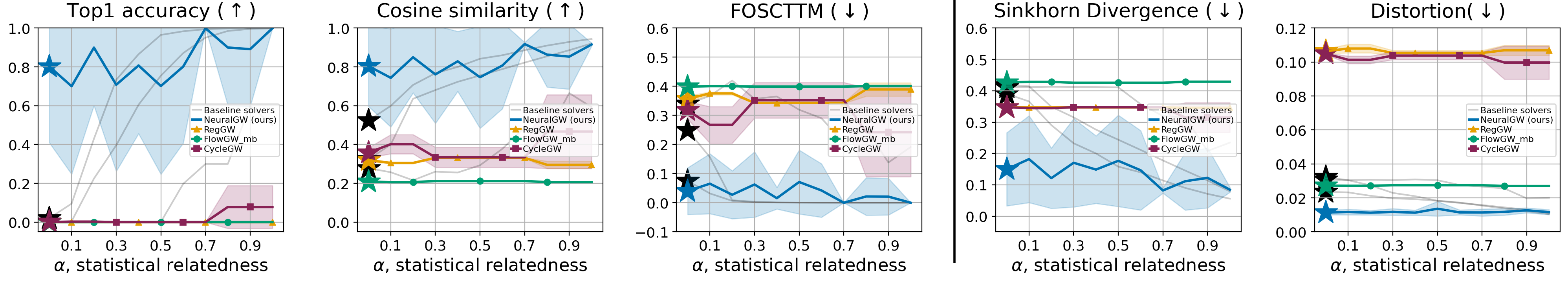} 
    \caption{\small MUSE-BPEmb (English) dataset ($\text{Source:}100\rightarrow \text{Target:}50$)}
    \label{fig:continuous_MUSE_BP_100_50}
    \end{subfigure}
    
    \begin{subfigure}{1\textwidth}
    \centering
    \includegraphics[width=0.95\linewidth]{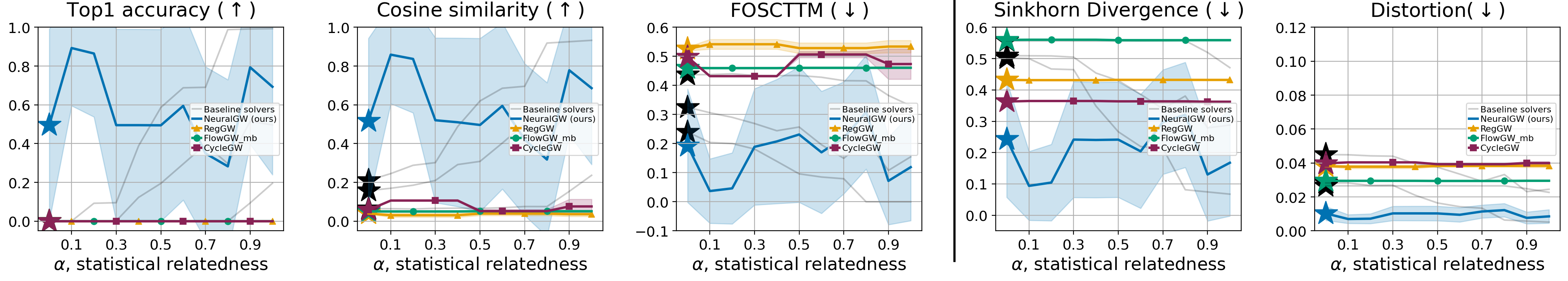} 
    \caption{\small VoxCeleb2-WNTA dataset ($\text{Source:}128\rightarrow \text{Target:}64$)}
    \label{fig:continuous_voxceleba2_wnta_128_64}
    \end{subfigure}
    
    \caption{\small Performance of the GWOT solvers for the large-scale \textbf{Twitter-GloVe} (\ref{fig:continuous_twitter_glove_100_50}) \textbf{MUSE-BPEmb (English)} (\ref{fig:continuous_MUSE_BP_100_50}) and \textbf{VoxCeleb2-WNTA} (\ref{fig:continuous_voxceleba2_wnta_128_64}) embeddings at different levels of statistical relatedness $\alpha$. The solvers were trained with $N_{\text{train}}/2=198\text{K}$, $N_{\text{train}}/2=44\text{K}$ and $N_{\text{train}}/2=298\text{K}$ samples, respectively. The plot shows results on a 2048-sample test subset. Accuracy was computed on the \textit{whole space} of $N_{\text{train}} \!+ \!N_{\text{test}}$ of $400\text{K}$, $90\text{K}$ and $600\text{K}$ samples, respectively.}
    \label{fig:metrics_neuralgw}
\end{figure*}

The methods under comparison for the \textbf{large-scale} setups are: NeuralGW (\S \ref{subsec:neuralgw}); RegGW; CycleGW and FlowGW (\S \ref{sec:related_work}). The latter is trained in a minibatch manner, i.e., it fits flow matching on top of discrete GW solutions for minibatches. For completeness, we additionally report the performance of the baselines from \S \ref{subsec:benchmark} (in gray, labeled as ``Baseline solvers'', Figure \ref{fig:metrics_neuralgw}). In this plot, the baseline methods are trained on a small subset ($N_{train} = 6$K), but the metrics are computed with respect to the \textit{whole} data space, similar to NeuralGW, RegGW, CycleGW and FlowGW.

\ul{The colored charts for baselines' metrics} computed on the \textit{whole} data spaces of Twitter-GloVe, MUSE-BPEmb, and VoxCeleb2-WNTA are shown in Figure~\ref{fig:small_train-large_test} (Appendix~\ref{ap:experiments-Glove}). While comparing baseline methods trained on $3$K samples (shown in gray) with our approach trained on larger datasets ($200$K for Twitter, $90$K for MUSE, and $300$K for VoxCeleb2) may seem unfair, we emphasize that the dataset sizes were selected according to the computational capabilities of each solver.

\newpage Based on the results in Figure \ref{fig:metrics_neuralgw}, we note the following:
\vspace{-1mm}
\begin{enumerate}[leftmargin=*]
    \item NeuralGW is the only method that handles statistically unrelated setups ($\alpha=0$).  In spite of the hard optimization problem, NeuralGW succeeds in finding a GWOT map in 7 out of 10 repeats.
    
    \vspace{-1mm}\item The advanced baselines (RegGW, CycleGW and FlowGW) failed to achieve reasonable performance for any value of $\alpha$. This happens because, although the source and target datasets are globally statistically related, the random permutation applied to the target data (see \S\ref{subsec:correlated_setups}) leads to batches with little or no statistical relation. Consequently, these models are trained on batches with $\alpha_{\text{batch}} \approx 0$. As these methods somehow rely on discrete methods when finding the GW coupling for every batch, they inherit this limitation.
\end{enumerate}

\vspace{-1mm}
For the sake of completeness, we provide some additional experiments and metrics in Appendix \ref{ap:experiments}: 

\vspace{-3mm}
\begin{itemize}[leftmargin=*]
    \item \ul{NeuralGW trained on $3K$ samples} (Appendix \ref{ap:experiments-Glove}). The results are bad, which is expected considering NeuralGW requires larger datasets due to adversarial training.
    
    \vspace{-1mm}
    \item \ul{Statistically related batches} (Appendix \ref{ap:paired_batches}). We train RegGW, CycleGW and FlowGW (minibatch) on  statistically related batches, then $\alpha_{\text{batch}}=\alpha$. This improves the performance; however, this setup is unfair due to this artificial construction, and the results should be interpreted with caution. 
    \vspace{-1mm}
    
    \item \ul{Low-to-high experiments} (Appendix \ref{ap:experiments-Glove}). While the main text focuses on settings where the source dimension exceeds the target dimension ($d_x>d_y$), we also consider the reverse case ($d_x<d_y$) for completeness. Results for the Twitter-GloVe $50\rightarrow100$ setup using baseline and large-scale solvers are presented in Figures \ref{fig:continuous_twitter_50_100} and \ref{fig:discrete_twitter_25_50_400K}, respectively. NeuralGW fails in this setting because a regular deterministic neural map $T$ generally cannot generate the additional degrees of freedom required to satisfy $T_{\sharp}\mu=\nu$ when the target distribution is full-dimensional.

\end{itemize}

\vspace{-2mm}
\textbf{Conclusions.} As shown in Figures~\ref{fig:metrics_neuralgw} and \ref{fig:discrete_main_text}, neither the baseline methods nor the proposed solver (NeuralGW) can fully handle statistically unrelated setups. A key limitation of the baseline solvers (\S\ref{subsec:benchmark}) is their reliance on discrete GW solvers: while effective on small datasets, these approaches fail in statistically unrelated settings. Learning the mapping in such scenarios requires training on large scale datasets to properly capture the intrinsic geometries. We therefore examine more advanced solvers, which nevertheless also fail due to their dependence on discrete solvers at some stage of the learning process. Finally, we propose NeuralGW, a fully continuous method that avoids discrete solvers and is the only solver capable of successfully handling the statistically unrelated setup.

\vspace{-1mm}

\vspace{-1mm}
\section{Discussion}
\vspace{-1mm}
The main goal of this paper is to highlight a previously overlooked issue in the evaluation pipelines of GWOT solvers. In particular, statistically unrelated setups have not been explored in prior work, despite representing a potentially more impactful application of GWOT methods, as they introduce a clear separation between data alignment and matching problems. Besides, it was recently noticed that the problem of correct evaluation pipeline does affect the related fields of Generative Modeling, namely, population dynamics with backward Euler scheme (JKO) \citep[Appendix A.3]{persiianov2026learning} which underlines the importance of the presented study.

\section*{Acknowledgments}
\vspace{-2mm}
The work was supported by the grant for research centers in the field of AI provided by the Ministry of Economic Development of the Russian Federation in accordance with the agreement 000000C313925P4F0002 and the agreement №139-10-2025-033. %

\bibliographystyle{icml2026}

\bibliography{example_paper}

\newpage

\appendix
\onecolumn %

\addtocontents{toc}{\protect\setcounter{tocdepth}{3}}
\tableofcontents

\newpage

\section{Proofs of theorems and lemmas}
\label{ap:proofs}

\textbf{Proof of Lemma \ref{lem:innergw}.} First, we recall the dual  formulation of \eqref{eq:monge-ot}, see, e.g., \citep{fan2023neural}:
\begin{equation}\label{eq:dualOT}
  \text{OT}(\Pd,\Qd) \eqdef \max_f \left[ \min_T \int \big(c(\mathbf{x},T(\mathbf{x})) -f(T(\mathbf{x}))\big)d\Pd(\mathbf{x}) + \int f(\y)d\Qd(\y)\right],
\end{equation}

respectively. Note that the existence of a solution $(f^*, T^*)$ of \eqref{eq:dualOT} is due to \citep[Theorem 2]{fan2023neural}. Our proof starts with \eqref{eq:maxot}. For each $\textbf{P}$, we rewrite the inner optimization over $\textbf{P}$ using \eqref{eq:dualOT} for the cost $c(\mathbf{x},\y)=-\<\textbf{P}\mathbf{x},\y\>$:

\begin{eqnarray}\label{eqn:lemma_proof}
    \nonumber \op{Const}{(\mathbb{P},\mathbb{Q})}-\text{innerGW}^2_2(\Pd,\Qd) = \max_{\textbf{P} \in F_{d_x, d_y}}\max_{\pi \in \Pi(\Pd, \Qd)} \int\limits_{\Rdx \times \Rdy}  \langle\textbf{P}\mathbf{x}, \y\rangle_{d_y} d\pi(\mathbf{x}, \y) = \\ \nonumber
     - \min_{\textbf{P} \in F_{d_x,d_y}} \left[\min_{\pi \in \Pi(\Pd,\Qd)}  \int\limits_{\Rdx \times \Rdy} -\langle\textbf{P}\mathbf{x}, \y\rangle_{d_y} d\pi(\mathbf{x}, \y)\right] = \label{eqn:lemma_proof1}\\
    \nonumber - \min_{\textbf{P} \in F_{d_x, d_y}} \left[ \max_f \int\limits_{\Rdx} f(\y) d\Qd(\y) + \min_{T \colon \Rdx \to \Rdy} \int\limits_{\Rdx} -\langle\textbf{P}\mathbf{x}, T(\mathbf{x})\rangle_{d_y}- f(T(\mathbf{x}))d\Pd(\mathbf{x}) \right] = \\
    \nonumber
    - \min_{\textbf{P} \in F_{d_x, d_y}}\max_f \min_{T} \mathcal{L}(\textbf{P}, f, T).
\end{eqnarray}

\textbf{Proof of Theorem \ref{thm:GW}}. We expand $\mathcal{L}(\textbf{P}^*, f^*, T^*)$ and use the fact that $T^*$ is an optimal GW map.

\begin{equation}\label{eqn:line1}
    \mathcal{L}(\textbf{P}^*, f^*, T^*) = \int\limits_{\Rdy} f^*(\y) d\Qd(\y) - \int\limits_{\Rdx} \left[\langle\textbf{P}^* \mathbf{x},T^* (\mathbf{x})\rangle_{d_y} + f^* \big(T^* (\mathbf{x})\big)\right]d\Pd(\mathbf{x}).
\end{equation}

Since $T^*_\sharp\Pd = \Qd$, by the change of variables formula we get:

\begin{equation*}
    \int\limits_{\Rdx} f^*\big(T^*(\mathbf{x})\big)d\Pd(\mathbf{x}) = \int\limits_{\Rdy} f^*(\y) d\Qd(\y).
\end{equation*}

Plugging this into \eqref{eqn:line1}, we get:

\begin{equation*}
    \mathcal{L}(\textbf{P}^*, f^*, T^*) = -\int\limits_{\Rdx} \langle\textbf{P}^* \mathbf{x},T^* (\mathbf{x})\rangle_{d_y} d\Pd(\mathbf{x}).
\end{equation*}

Since $T^*$ is an optimal GW map, the induced coupling $\pi^* = [\text{id}_{\Rdx}, T^*]_\sharp\Pd$ is GW-optimal. By the equivalence between the innerGW and max-correlation formulations, and the choice of $\textbf{P}^*$ in Theorem \ref{thm:GW}, the pair $(\textbf{P}^*,\pi^*)$ maximizes the correlation objective. Consequently, $\textbf{P}^*$ solves the outer minimization in \eqref{eq:neuralgw}, and $\pi^*$ is also optimal for the fixed matrix $\textbf{P}^*$.

Finally, since $f^*$ is an optimal dual potential for the corresponding OT problem, by \eqref{eqn:lemma_proof} we have:

\begin{eqnarray}
    \nonumber -\mathcal{L}(\textbf{P}^*, f^*, T^*) &=& \int\limits_{\Rdx \times \Rdy} \langle\textbf{P}^* \mathbf{x},\y\rangle_{d_y} d\pi^* (\mathbf{x}, \y) \\
    \nonumber &=& \max_{\pi \in \Pi(\Pd, \Qd)} \int\limits_{\Rdx \times \Rdy} \langle\textbf{P}^*\mathbf{x}, \y\rangle_{d_y} d\pi(\mathbf{x}, \y) \\
    \nonumber &=& -\min_{T \colon \Rdx \to \Rdy}\mathcal{L}(\textbf{P}^*, f^*, T),
\end{eqnarray}

which completes the proof.
    
\section{Toy sanity-check experiment}\label{ap:toy3Dto2D}
\vspace{-2mm}
To illustrate the capabilities of the solvers and as a necessary sanity check for the implementations, we propose a toy experiment. In this setup, the source distribution is a mixture of Gaussians in $\mathbb{R}^3$ and the target is also a mixture of Gaussians but in $\mathbb{R}^2$, see Figure \ref{fig:toysrctgt}. By choosing this experiment on incomparable spaces, we ensure the solvers are actually capable of dealing with a real Gromov-Wasserstein problem. The obtained results for the baselines can be found in Figure \ref{fig:toydiscrete} and \ref{fig:regGW}, Figure \ref{fig:toyneuralv2} shows the result for our method, NeuralGW. As we can see, for all methods a component of the source distribution is mostly mapped to a component/neighbouring components of the target distribution, indicating the correct GW alignment. We also report the distortion $d$ obtained after the methods were trained, the ground truth value for Figure \ref{fig:toysrctgt} was computed using the discrete GW solver from the POT library.

\begin{figure*}[h]
    \centering
    \minipage{0.205\textwidth}
        \begin{subfigure}{1\textwidth}
      \includegraphics[width=\linewidth]{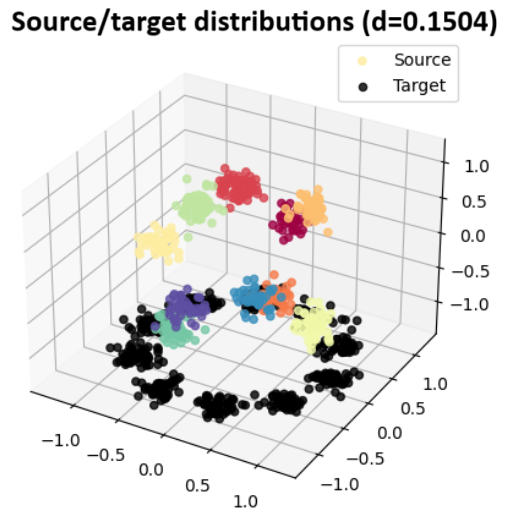}
      \vspace*{-5mm}
      \caption{\fontsize{8}{12}\selectfont Source, target distributions.}\label{fig:toysrctgt}
      \end{subfigure}
    \endminipage%
    \hspace*{2mm}\minipage{0.192\textwidth}
        \begin{subfigure}{1\textwidth}
    
      \includegraphics[width=\linewidth]{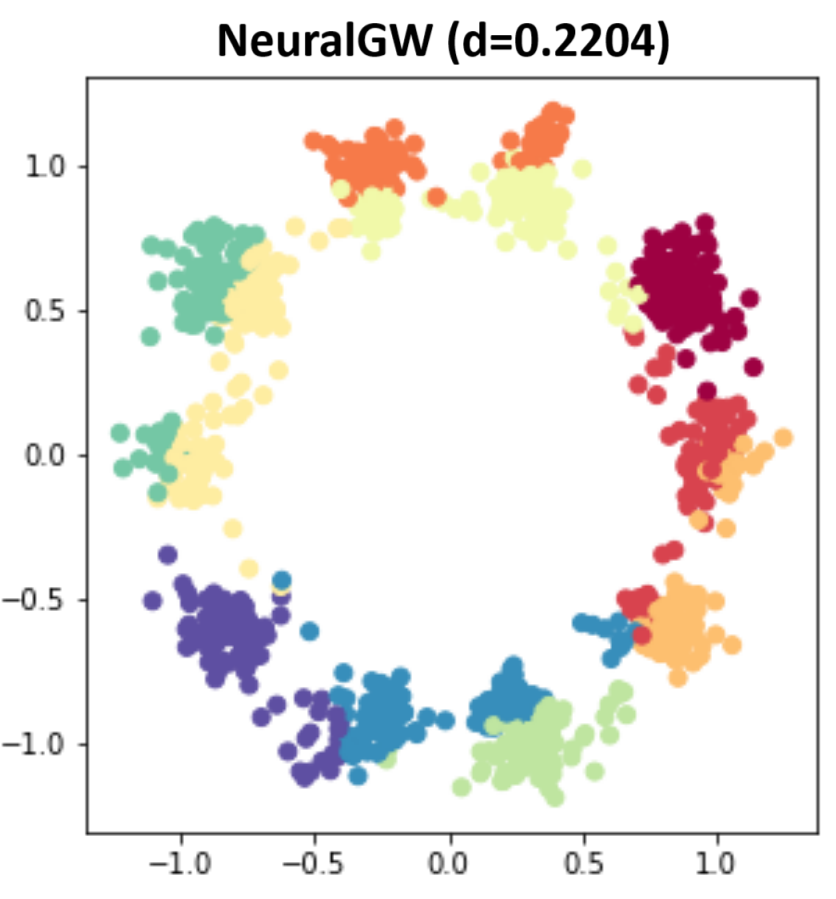}
      \caption{\fontsize{8}{12}\selectfont NeuralGW solver (ours)}\label{fig:toyneuralv2}
      \end{subfigure}
    \endminipage%
    \hspace*{2mm}\minipage{0.46\textwidth}%
    \centering
        \begin{subfigure}{1\textwidth}
    
      \includegraphics[scale=0.35]{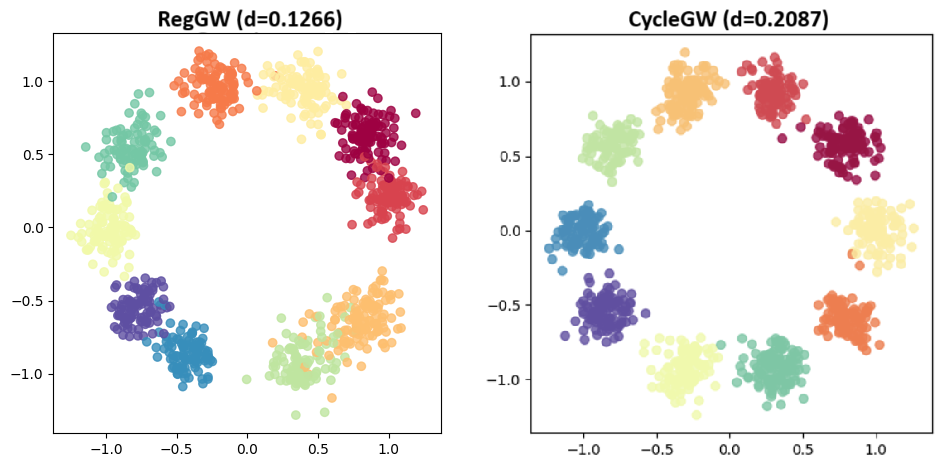}
      \caption{\fontsize{8}{12}\selectfont Batched solvers (RegGW and CycleGW)}\label{fig:regGW}
      \end{subfigure}
      
    \endminipage
    \vspace*{2mm}
    \begin{subfigure}{1\textwidth}
    \centering
    \hspace*{-3mm}\includegraphics[scale=0.3]{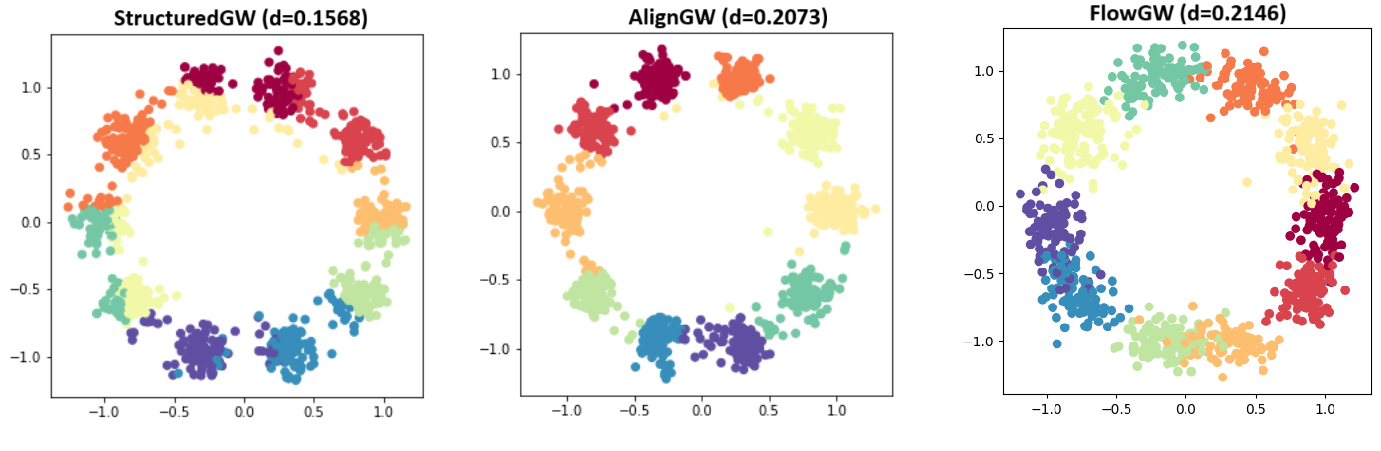}
    \vspace{-2mm}\caption{ \fontsize{8}{12}\selectfont Baseline solvers (StructuredGW, AlignGW, FlowGW).}
    \label{fig:toydiscrete}
    \end{subfigure}
    
    \vspace{-2mm}\caption{Learned GWOT map $T$ by different solvers; Toy (3D$\rightarrow$2D) experiment.}
    \label{fig:toy}
    \vspace{-4mm}
\end{figure*}

\section{Metrics and embeddings}\label{ap:metrics_embeddings}
\vspace{-2mm}
\subsection{Metrics}\label{ap:metrics_theory}

This subsection explains how our evaluation metrics are computed. In the main text, we report what we refer to as \textit{matching metrics}. These include Top-$k$ accuracy, cosine similarity, and the Fraction of Samples Closer Than the True Match (FOSCTTM). Additionally, in Appendix \ref{ap:experiments}, we report the \textit{marginal metrics}, Maximum Mean Discrepancy (MMD) \citep{JMLR:v13:gretton12a}, Sinkhorn divergence \citep{feydy2018interpolatingoptimaltransportmmd}, Bures-Wasserstein unexplained variance percentage ($\text{B}\mathbb{W}_2^2\text{-UVP}$) \citep{korotin2019wasserstein} and distortion. All the metrics (with exception of FOSCTTM) are computed on normalized data, therefore, $\tilde{\mathbf{x}}=\mathbf{x}/\|\mathbf{x}\|_2$ where $\|\cdot\|_2$ denotes the Euclidean vector norm. Below we provide a concise explanation about how they are computed as well as the motivation to pick them.

\textbf{Matching metrics}

These metrics aim to provide an intuitive assessment of the quality of the predicted alignments by measuring how well the recovered correspondences match the ground truth. In order to compute them, it is necessary to know the actual ground-truth pairs.

\begin{itemize}[leftmargin=*]

    \item \textbf{Top $k$-accuracy} ($\uparrow$). Let $\mathbf{x}_{\text{test}} \in \mathbb{R}^{m \times d_x}$ be the $m$ test samples from the source domain and $\mathbf{y}_{\text{pred}} \in \mathbb{R}^{m \times d_y}$ their predicted counterparts. For each $\hat y_j$ in $\mathbf{y}_{\text{pred}}$, we use \texttt{faiss} to retrieve the $k$ nearest neighbors from the \textit{reference} spaces and denote their indices by $\{c_{j,1}, \dots, c_{j,k}\}$. Given the ground-truth labels $\{l_1, \dots, l_m\}$ for the test samples, the Top-$k$ accuracy is defined as
\[
\text{Top }k
= \frac{1}{m}\sum_{j=1}^m \mathbbm{1}\{\, l_j \in \{c_{j,1}, \dots, c_{j,k}\}\,\}.
\]

    The size of the \textit{reference} space may vary depending on the specific experiment, as discussed in the main text. This metric is particularly useful due to its intuitive interpretability. Additionally, evaluating performance across different values of $k$ in Top-$k$ accuracy offers insight into how well the solver reconstructs the target distribution while preserving its internal structure.

    \item \textbf{Fraction of Samples Closed Than the True Match (FOSCTTM).} ($\downarrow$) For each ground-truth pair $(x_i, y_i)$, let $\hat y_i = T(x_i)$ be the transported point and let $d(\cdot,\cdot)$ denote the Euclidean distance. We compute the proportion of target samples $y_j$ that satisfy $d(\hat y_i, y_j) < d(\hat y_i, y_i)$. Formally,
    \[
    \mathrm{FOSCTTM}(i)
    = \frac{1}{N}
    \left|\left\{ j : d(\hat y_i, y_j) < d(\hat y_i, y_i) \right\}\right|.
    \]
    The overall metric is the average $\mathrm{FOSCTTM} = \frac{1}{N} \sum_{i=1}^N \mathrm{FOSCTTM}(i)$. A perfect alignment yields $\mathrm{FOSCTTM}=0$. We note that this metric is insensitive to the choice of \textbf{reference} space, as it depends solely on distances to true paired samples.

    \item \textbf{Cosine similarity} ($\uparrow$). It is computed between the predicted vector and the optimal pair vector in the target \textit{reference} space. 
\end{itemize}

\textbf{Marginal metrics}

We additionally report several metrics to ensure that the learned transport not only aligns individual samples but also accurately matches the overall distributions. The metrics considered are the Maximum Mean Discrepancy (MMD) \citep{JMLR:v13:gretton12a}, the Sinkhorn divergence \citep{feydy2018interpolatingoptimaltransportmmd}, and the Bures-Wasserstein unexplained variance percentage \citep{korotin2019wasserstein}. The metrics reported are computed for the testing data only. These metrics are computed between the empirical pushforward of the source distribution, \(T_{\sharp}\hat{\Pd}\eqdef \frac{1}{N_x}\sum_{i=1}^{N_x} \delta(T(\mathbf{x}_i))\approx\hat{\Qd}\) (i.e., the output of the learned transport map or coupling) and the actual target empirical distribution, $\hat{\Qd}$, to assess how well the learned mapping aligns the two distributions. However, these metrics do not directly reflect the quality of the solver itself, as they capture only distribution-level alignment and not, for example, matching or correspondence-based fidelity.

\begin{itemize}[leftmargin=*]
\vspace{-2mm}
    \item \textbf{Maximum Mean Discrepancy (MMD)} ($\downarrow$). It measures how different two distributions are by comparing their average representations in a high-dimensional feature space defined by a kernel. We denote $\mathbf{z}=T(\mathbf{x})$, therefore $\mathbf{z}\sim T_{\sharp}\Pd$. In our setup we used a Gaussian RBF kernel: 
    
    \vspace{-2mm}
    $$
    \text{MMD}^2(T_{\sharp}\Pd, \Qd) = \mathbb{E}_{\mathbf{z}, \mathbf{z}' \sim T_{\sharp}\Pd}[k(\tilde{\mathbf{z}}, \tilde{\mathbf{z}}')] + \mathbb{E}_{\y, \y' \sim \Qd}[k(\tilde{\y}, \tilde{\y}')] - 2 \mathbb{E}_{\mathbf{z} \sim T_{\sharp}\Pd, \y \sim \Qd}[k(\tilde{\mathbf{z}}, \tilde{\y})],
    $$

    where $k(\mathbf{x}, \y) = \exp\left(-\frac{\|\mathbf{x} - \y\|^2}{2 \sigma^2}\right)$ , we set $\sigma = 10$.

    \item \textbf{Sinkhorn Divergence} ($\downarrow$). It quantifies the discrepancy between two probability distributions, $T_{\sharp}\Pd$ and $\Qd$ in our case.
    
    \vspace{-2mm}
    $$
    S_\varepsilon(T_{\sharp}\Pd, \Qd) = \mathrm{OT}_\varepsilon(T_{\sharp}\Pd, \Qd) - \frac{1}{2}\mathrm{OT}_\varepsilon(T_{\sharp}\Pd, T_{\sharp}\Pd) - \frac{1}{2}\mathrm{OT}_\varepsilon(\Qd, \Qd),
    $$
    where $\mathrm{OT}_\varepsilon$ is the entropy-regularized optimal transport cost with regularization strength $\varepsilon > 0$. In our computations we set $\varepsilon=0.01$.

    \item \textbf{Bures-Wasserstein unexplained variance percentage ($\text{B}\mathbb{W}^2_2\text{-UVP}$)} ($\downarrow$). The $\text{B}\mathbb{W}^2_2\text{-UVP}$ metric measures how much of the target distribution’s covariance is not captured by the pushforward of the source distribution.
    $$
    \text{B}\mathbb{W}_2^2\text{-UVP} \left( T \sharp \Pd, \Qd \right) = \frac{100\%}{\frac{1}{2} \, \text{Var}(\Qd)} \, \text{B}\mathbb{W}_2^2 \left( T \sharp \Pd, \Qd \right)
    $$
    where \( \text{B}\mathbb{W}_2^2 \left( T \sharp \Pd, \Qd \right) \) is the Bures–Wasserstein metric between the Gaussian approximations of the distributions \( T \sharp \Pd \) and \( \Qd \).

    \item \textbf{Distortion ($\downarrow$).} This metric is related to the Gromov-Wasserstein cost introduced in \eqref{eq:gwot-problem}. 
    $$
    \text{D}^2_2(\mathbb{P}, T) \eqdef 
     \!\!\!\!\!\!\!\int\limits_{\Rdx\times\Rdx}\!\!\!\!\!\!\!\! \left(c_{\X}(\mathbf{x}, \mathbf{x}')\!-\!c_{\Y}(T(\mathbf{x}), T(\mathbf{x}'))\right)^2\!\!d\Pd(\mathbf{x})d\Pd(\mathbf{x}').
    $$
    In practice we compute it for the test samples according to the following expression: 
    $$
    \text{D}(\tilde{\mathbf{x}}, T)\eqdef 
      \frac{1}{n^2}\sum_{i,j=1}^{n}\left(c_{\X}(\tilde{\mathbf{x}}_i, \tilde{\mathbf{x}}_j)\!-\!c_{\Y}(T(\tilde{\mathbf{x}}_i), T(\tilde{\mathbf{x}}_j))\right)^2.
    $$
    
    \vspace{-2mm}
    
    As this metric depends on the type of cost functions, we set $c_\X=c_\Y=c$ to be the cosine distance, $c(\tilde{\cdot}, \tilde{\cdot}' )=1-\text{cos-sim}(\tilde{\cdot}, \tilde{\cdot} ')$. This choice of $c$, along with the normalization of input vectors to unit norm, ensures consistency across methods. In particular, approaches such as StructuredGW and NeuralGW employ the inner product as their intra-domain cost function, whereas others directly use the cosine distance. By normalizing and adopting a cosine-based cost, we unify the representation and make the comparisons between methods more meaningful.
\end{itemize}

\vspace{-2mm}
\subsection{Embeddings}\label{ap:embeddings_theory}
\vspace{-2mm}

\textbf{GloVe embeddings and datasets}

The GloVe embeddings of words are generated using the GloVe algorithm \citep{pennington2014glove}. Its main advantage is that the embedded vectors capture semantic relationships and exhibit linear substructures in the vector space. This allows meaningful computation of distances and alignments, which is central to the GWOT framework. The authors provide access to the GloVe embeddings of four datasets: Wikipedia, Gigaword, Common Crawl, and Twitter. We focus solely on the GloVe embeddings for the Twitter corpus. As we mentioned in the main text, we randomly take $400$K samples from a total of $1.2$ million, we use the pre-trained embeddings provided by the authors.

\textbf{Byte-Pair embeddings and datasets}

The original MUSE embeddings were obtained using fastText \citep{joulin2016fasttext}; however, most methods considered in this paper struggle to align them effectively. Although AlignGW \citep{alvarez2018gromov} reports positive results on these embeddings, its evaluation relies on cross-domain similarity local scaling (CSLS) \citep{DBLP:journals/corr/abs-1710-04087}. CSLS is commonly used to improve multilingual alignment, but introduces a corrective bias that may overestimate the intrinsic alignment capabilities of the solver. Therefore, we consider an alternative approach for embedding the MUSE vocabularies.

We construct new embeddings using the bpemb library\footnote{https://github.com/bheinzerling/bpemb/tree/master}, which provides pretrained subword embeddings in different dimensions. We select the largest available vocabulary (200K) to maximize coverage of the MUSE dictionaries, excluding words without corresponding embeddings. We consider English and Spanish, retaining only words with unique translations to ensure that the resulting datasets satisfy our definition of statistical relatedness in Section \wasyparagraph\ref{subsec:correlated_setups}. Further details on BPEmb can be found in the original paper.

Following this procedure, we construct BPEmb datasets for the MUSE (English and Spanish) and Twitter vocabularies. Although filtering reduces the number of samples, the resulting datasets remain sufficiently large for training continuous methods.

\textbf{WNTA embeddings and datasets}

The VoxCeleb2 dataset \citep{Chung18b} consists of over one million utterances from more than 6,000 speakers, making it one of the largest publicly available collections for speaker recognition research. For feature extraction, we adopt a pretrained WavLM model \citep{Chen2021WavLM}, a large-scale self-supervised architecture that jointly learns masked speech prediction and denoising. Following \citet{gengembre24_interspeech}, we use the WNTA (Weighted Non‑Timbral Aggregation) variant of this model to produce embeddings that capture non‑timbral voice traits (prosody, accent, etc.). The model is available on HuggingFace\footnote{\url{https://huggingface.co/Orange/Speaker-wavLM-pro}} and comprises a pretrained WavLM model followed by a custom top‑layer block that performs statistics pooling and dimensionality reduction. We generate embeddings of two different dimensions, 128 and 64, and denote them as VoxCeleb2‑WNTA. We consider only the first $600\text{K}$ samples.

\section{Extended experiments}\label{ap:experiments}

\vspace{-1mm}

\textbf{Overview of the conducted experiments.} To help the reader navigating over all our considered experiments, We provide Table \ref{table:experiments_summary} summarizing the full list of experiments in our paper.

\vspace{3mm}
\begin{table}[!hbt]
\centering
\begin{tblr}{
    colspec={c|c|c|c|c|c},
    vline{2-6}={solid},
    hline{1,2,7,10,12,14,17}={solid},
    hline{3,15}={6}{solid},
    hline{4,5,6,8,9}={5-6}{solid},
    hline{11,16}={4-6}{solid},
    hline{13}={4-5}{solid},
    rowsep=4pt,
    colsep=5pt,
    cells={halign=c,valign=m}
}

\textbf{\shortstack{Dataset\\name}} &
\textbf{Embedding} &
\textbf{\shortstack{Total\\samples}} &
\textbf{Train/test size} &
\textbf{\shortstack{Source\\$\rightarrow$\\Target}} &
\textbf{Evaluated on} \\

\SetCell[r=5]{c} Twitter &
\SetCell[r=5]{c} GloVe &
\SetCell[r=5]{c} 400K &
\SetCell[r=5]{c} \shortstack{Baseline solvers\\6000/2048} &
\SetCell[r=2]{c} 100$\rightarrow$50 &
\shortstack{8048 samples\\Figure \ref{fig:discrete_main_text}} \\

& & & & &
\shortstack{400K samples\\Figure \ref{fig:discrete_twitter_glove_100_50_400K}} \\

& & & & 50$\rightarrow$25 &
\shortstack{400K samples\\Figure \ref{fig:discrete_twitter_50_25_400K}} \\

& & & & 50$\rightarrow$100 &
\shortstack{400K samples\\Figure \ref{fig:discrete_twitter_50_100_400K}} \\

& & & & 25$\rightarrow$50 &
\shortstack{400K samples\\Figure \ref{fig:discrete_twitter_25_50_400K}} \\

\SetCell[r=3]{c} Twitter &
\SetCell[r=3]{c} GloVe &
\SetCell[r=3]{c} 400K &
\SetCell[r=3]{c} \shortstack{Continuous solvers\\360K/2048} &
100$\rightarrow$50 &
\shortstack{400K samples\\Figure \ref{fig:continuous_twitter_glove_100_50}} \\

& & & & 50$\rightarrow$25 &
\shortstack{400K samples\\Figure \ref{fig:continuous_twitter_50_25}} \\

& & & & 50$\rightarrow$100 &
\shortstack{400K samples\\Figure \ref{fig:continuous_twitter_50_100}} \\

\SetCell[r=2]{c} MUSE &
\SetCell[r=2]{c} Byte-Pair &
\SetCell[r=2]{c} 90K &
\shortstack{Baseline solvers\\6000/2048} &
\shortstack{100(English)\\$\rightarrow$\\50(English)} &
\shortstack{90K samples\\Figure \ref{fig:discrete_muse_bp_100_50_400K}} \\

& & & \shortstack{Continuous solvers\\88K/2048} &
\shortstack{100(English)\\$\rightarrow$\\50(English)} &
\shortstack{90K samples\\Figure \ref{fig:continuous_MUSE_BP_100_50}} \\

\SetCell[r=2]{c} MUSE &
\SetCell[r=2]{c} Byte-Pair &
\SetCell[r=2]{c} 60K &
\shortstack{Baseline solvers\\6000/2048} &
\shortstack{100(English)\\$\rightarrow$\\100(Spanish)} &
\SetCell[r=2]{c} \shortstack{60K samples\\Table \ref{table:muse_en_es_bp}} \\

& & & \shortstack{Continuous solvers\\58K/2048} &
\shortstack{100(English)\\$\rightarrow$\\100(Spanish)} & \\

\SetCell[r=3]{c} VoxCeleb2 &
\SetCell[r=3]{c} WavLM-WNTA &
\SetCell[r=3]{c} 600K &
\SetCell[r=2]{c} \shortstack{Baseline solvers\\6000/2048} &
\SetCell[r=2]{c} 128$\rightarrow$64 &
\shortstack{8048 samples\\Figure \ref{fig:discrete-voxceleba-wnta-128-64}} \\

& & & & &
\shortstack{600K samples\\Figure \ref{fig:small_train-large_test}} \\

& & & \shortstack{Continuous solvers\\598K/2048} &
128$\rightarrow$64 &
\shortstack{600K samples\\Figure \ref{fig:continuous_voxceleba2_wnta_128_64}} \\

\end{tblr}
\caption{Summary of experiments}
\label{table:experiments_summary}
\end{table}

\vspace{-2mm}

\subsection{Marginal metrics for main text experiments}\label{ap:marginal_metrics}
\vspace{-3mm}

Here we report marginal metrics for our main experiments in Figures \ref{fig:discrete_main_text} and \ref{fig:metrics_neuralgw}. The marginal metrics for the baselines can be found in Figure \ref{fig:discrete_marginals} and in Figure \ref{fig:continuous_marginals} for continuous solvers.

\begin{figure*}[!htb]

    \begin{subfigure}{1\textwidth}
        \centering
        \includegraphics[width=0.92\linewidth]{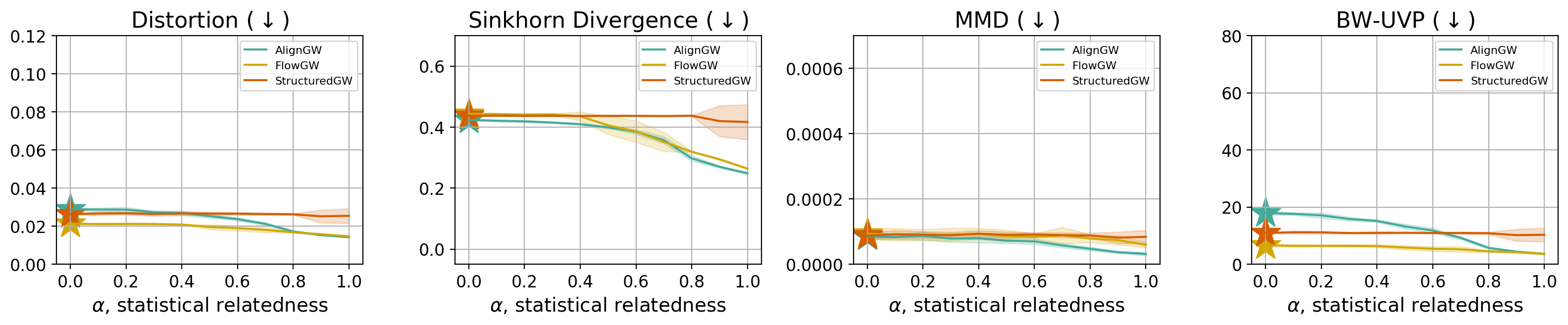} 
        \vspace{-2mm}
        \caption{\small Twitter-GloVe dataset ($\text{Source:}100\rightarrow \text{Target:}50$)}
        \label{fig:discrete_twitter_glove_100_50_marginals}
    \end{subfigure}

    \begin{subfigure}{1\textwidth}
        \centering
        \includegraphics[width=0.92\linewidth]{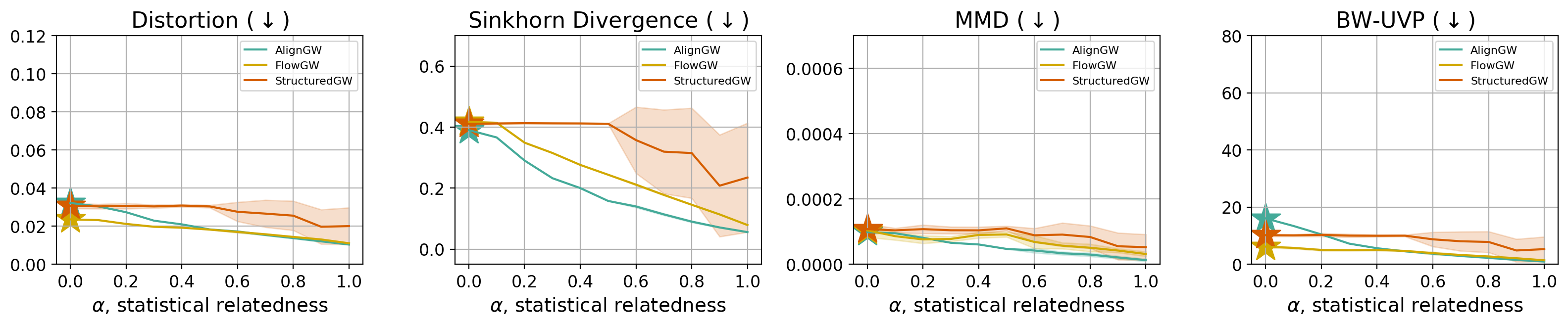} 
        \vspace{-2mm}
        \caption{\small MUSE-BPEmb (English) dataset ($\text{Source:}100\rightarrow \text{Target:}50$)}
        \label{fig:discrete_marginals_muse_bp_100_50}
    \end{subfigure}

    \begin{subfigure}{1\textwidth}
        \centering
        \includegraphics[width=0.92\linewidth]{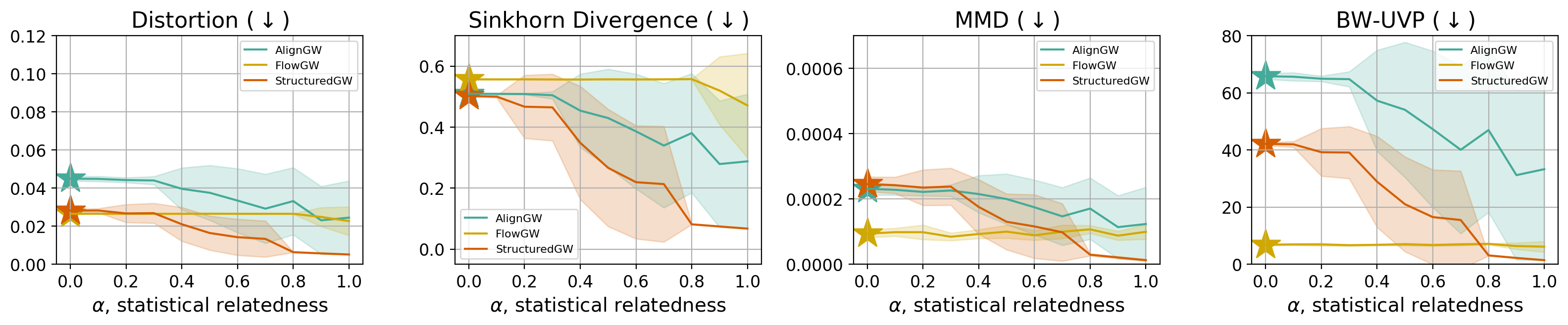} 
        \vspace{-2mm}
        \caption{\small VoxCeleb2-WNTA dataset ($\text{Source:}128\rightarrow \text{Target:}64$)}
        \label{fig:discrete_marginals_voxceleb2_wnta_128_64}
    \end{subfigure}

    \vspace{-2mm}
    \caption{\small Marginal metrics of baseline GWOT solvers for \textbf{Twitter-GloVe} (\ref{fig:discrete_twitter_glove_100_50_marginals}), \textbf{MUSE-BPEmb (English)} (\ref{fig:discrete_marginals_muse_bp_100_50}) and \textbf{VoxCeleb2-WNTA} (\ref{fig:discrete_marginals_voxceleb2_wnta_128_64}) at different levels of statistical relatedness $\alpha$. Solvers were trained with $N_{\text{train}}/2 = 3000$ samples from spaces of $400\text{K}$, $90\text{K}$ and $600\text{K}$, respectively. The plot shows results on a 2048-sample test subset. Accuracy was computed on the \emph{reference} space of $N_{\text{train}} \!+ \!N_{\text{test}} \!=\! 8048$ samples.}
    \label{fig:discrete_marginals}

    \vspace{2mm} %

    \begin{subfigure}{1\textwidth}
        \centering
        \includegraphics[width=0.92\linewidth]{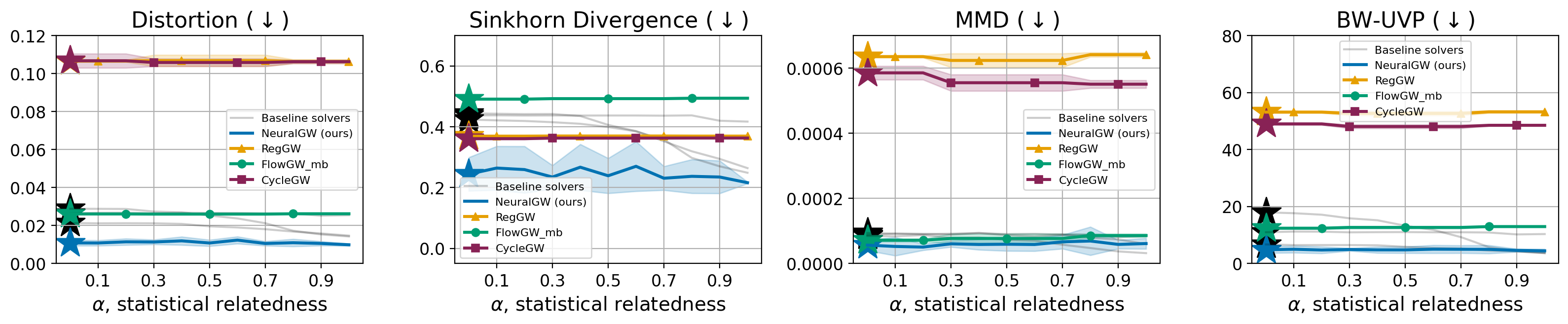} 
        \vspace{-2mm}
        \caption{\small Twitter-GloVe dataset ($\text{Source:}100\rightarrow \text{Target:}50$)}
        \label{fig:continuous_marginals_twitter_glove_100_50}
    \end{subfigure}

    \begin{subfigure}{1\textwidth}
        \centering
        \includegraphics[width=0.92\linewidth]{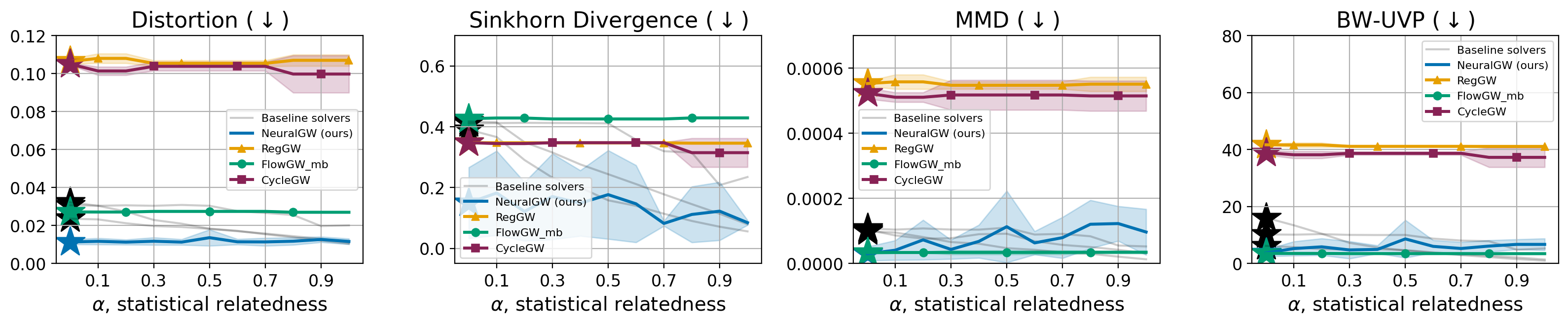} 
        \vspace{-2mm}
        \caption{\small MUSE-BPEmb (English) dataset ($\text{Source:}100\rightarrow \text{Target:}50$)}
        \label{fig:continuous_marginals_muse_bp_100_50}
    \end{subfigure}

    \begin{subfigure}{1\textwidth}
        \centering
        \includegraphics[width=0.92\linewidth]{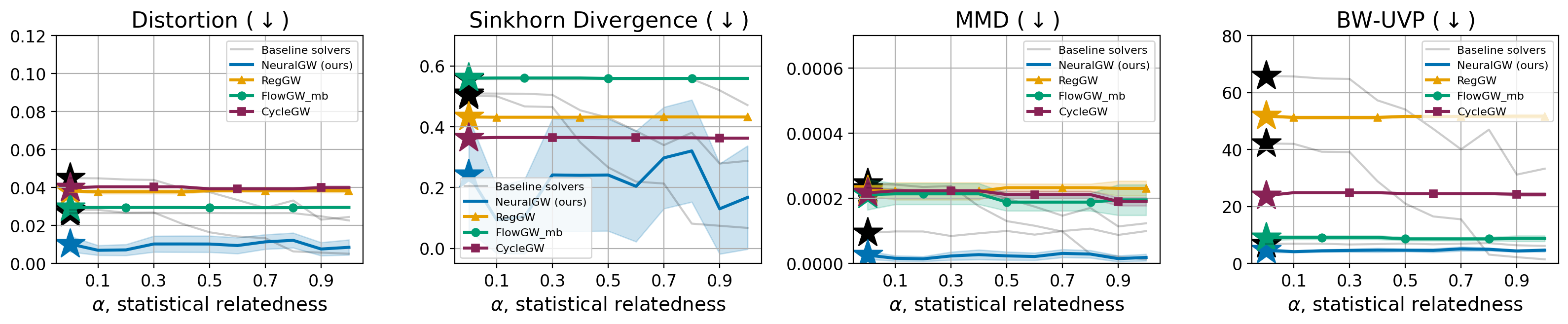} 
        \vspace{-2mm}
        \caption{\small VoxCeleb2-WNTA dataset ($\text{Source:}128\rightarrow \text{Target:}64$)}
        \label{fig:continuous_marginals_voxceleb2_wnta_128_64}
    \end{subfigure}

    \vspace{-2mm}
    \caption{\small Marginal metrics for the GWOT solvers in the large-scale \textbf{Twitter-GloVe} (\ref{fig:continuous_marginals_twitter_glove_100_50}), \textbf{MUSE-BPEmb (English)} (\ref{fig:continuous_marginals_muse_bp_100_50}) and \textbf{VoxCeleb2-WNTA} (\ref{fig:continuous_marginals_voxceleb2_wnta_128_64}) embeddings at different levels of statistical relatedness $\alpha$. The solvers were trained with $N_{\text{train}}/2=198\text{K}$, $N_{\text{train}}/2=44\text{K}$ and $N_{\text{train}}/2=288\text{K}$ samples, respectively. The plot shows results on a 2048-sample test subset. Accuracy was computed on the \textit{whole} data space of $N_{\text{train}} \!+ \!N_{\text{test}}$ of $400\text{K}$, $90\text{K}$ and $600\text{K}$ samples, respectively.}
    \label{fig:continuous_marginals}

\end{figure*}

\vspace{3mm}
\textbf{Conclusion.} We observe that marginal metrics correlate moderately with matching metrics, indicating that most methods align the marginal distributions to some extent. However, these metrics alone are insufficient to evaluate solver performance, as they ignore the structural alignment central to Gromov-Wasserstein objectives. While useful for assessing how target-like the transported samples are, marginal metrics should be considered complementary. Notably, our method, NeuralGW, consistently outperforms baselines in both structural and marginal alignment, and is the only approach capable of accurately matching marginals across all levels of statistical relation between source and target domains.

\vspace{-2mm}
\subsection{Trained on small dataset, evaluated w.r.t. \textit{whole} data space}\label{ap:main_experiments_whole_reference}
\vspace{-2mm}

The accuracy evaluation involves identifying the $k$-nearest neighbours for a given vector within a target vector space. As discussed in Section \wasyparagraph \ref{subsec:benchmark}, this space was limited to a small subset of $N_{train} + N_{test}$ samples for the experiments included in the main text, Figure \ref{fig:discrete_main_text}. This is justified since the methods were trained on a similar number of samples. However, evaluating accuracy across the entire data space offer insights into the models' ability to capture the structures of the probability distributions. Therefore, in Figure  \ref{fig:small_train-large_test} we show the Top-$k$ accuracies w.r.t. the \textit{whole} data space.

\vspace{-2mm}
\begin{figure}[h]
    \begin{subfigure}{1\textwidth}
    \centering
    \includegraphics[width=0.7\linewidth]{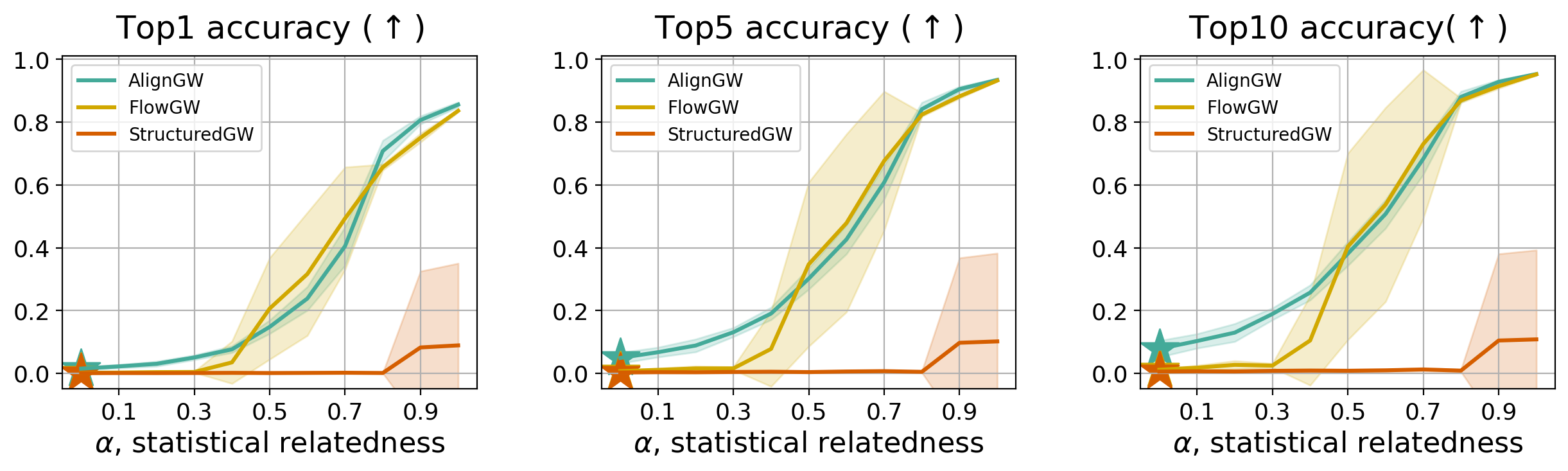} 
    \vspace{-2mm}
    \caption{\small Twitter-GloVe dataset ($\text{Source:}100\rightarrow \text{Target:}50$)}
    \label{fig:discrete_twitter_glove_100_50_400K}
    
    \hspace{2mm}
    \end{subfigure}

    \vspace{-2mm}
    \begin{subfigure}{1\textwidth}
    \centering
    \includegraphics[width=0.7\linewidth]{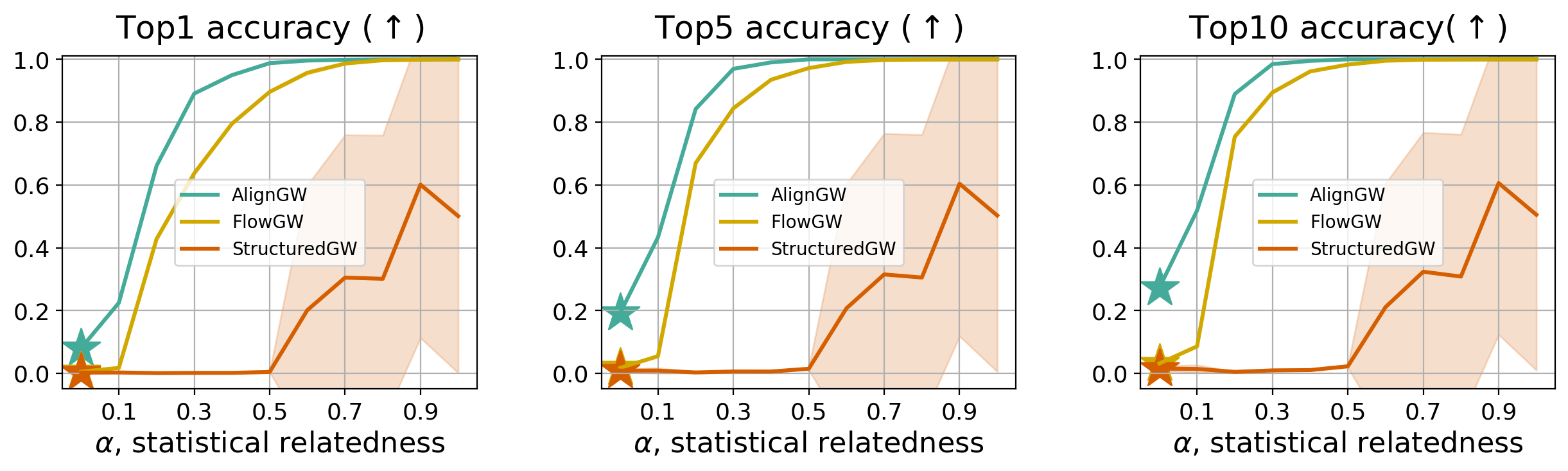} 
     \vspace{-2mm}
    \caption{\small MUSE-BPEmb (English) dataset ($\text{Source:}100\rightarrow \text{Target:}50$)}
    \label{fig:discrete_muse_bp_100_50_400K}
    
    \hspace{2mm}
    \end{subfigure}

    \vspace{-2mm}
    \begin{subfigure}{1\textwidth}
    \centering
    \includegraphics[width=0.7\linewidth]{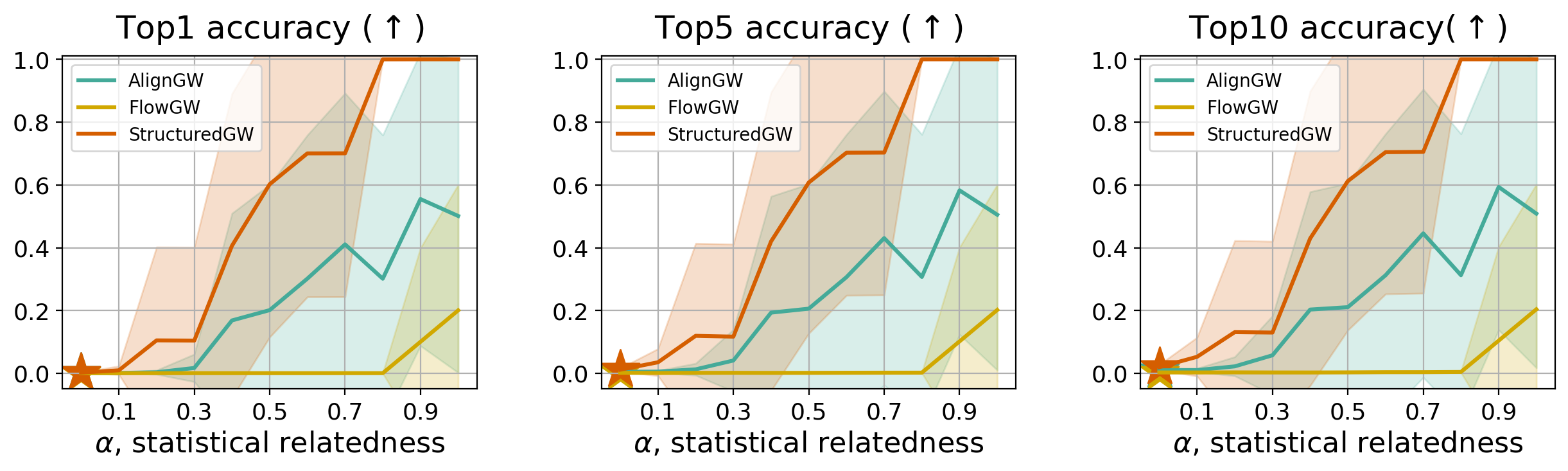} 
     \vspace{-2mm}
    \caption{\small VoxCeleb2-WNTA ($\text{Source:}128\rightarrow \text{Target:}64$)}
    \label{fig:discrete_voxceleb2wnta_128_64_600K}
    
    \hspace{2mm}
    \end{subfigure}
    
    \vspace{-5mm}\caption{\small Performance of the baseline GWOT solvers for \textbf{Twitter-GloVe} (\ref{fig:discrete_twitter_glove_100_50_400K}), \textbf{MUSE-BPEmb (English)} (\ref{fig:discrete_muse_bp_100_50_400K}) and VoxCeleb2-WNTA (\ref{fig:discrete_voxceleb2wnta_128_64_600K}) at different levels of statistical relatedness $\alpha$ for the $100 \rightarrow 50$ setup. Solvers were trained with $N_{\text{train}}/2 = 3000$ samples from spaces of 400K and 90K, respectively; the plot shows results on a 2048-sample test subset. Accuracy metrics were computed over the \emph{whole} data space of $400$K and $90$K samples, respectively.}
    \label{fig:small_train-large_test}
    \vspace{-3mm}
\end{figure}

\textbf{Conclusion.} The performance of the baseline solvers drops significantly when computing the accuracies on the \textit{whole} data space. This decline is expected, as these methods are trained using a limited subset of the data, which restricts their ability to generalize effectively to unseen samples. 

\vspace{-3mm}
\subsection{Experiments on statistically related batches for continuous solvers}\label{ap:paired_batches}

\vspace{-2mm}
As mentioned at the end of Subsection \S\ref{sec:exp_neuralgw}, we also report the performance of the continuous solvers when trained on statistically related batches under different levels of statistical relatedness ~$\alpha$. Figure \ref{fig:continuous_paired_twitter_glove_10_50} presents the results for one of our main experimental settings—Twitter-GloVe, from dimension 100 to 50—originally introduced in Figure \ref{fig:continuous_twitter_glove_100_50}. The total number of training samples remains fixed at $N_{\text{train}}/2 = 200\text{K}$, and we use batches of 500 samples across all methods. Further details on the hyperparameters are provided in Appendix \ref{ap:code}. We omit the MUSE-BPEmb case, as we expect qualitatively similar results.

\vspace{-1mm}
\begin{figure*}[h]

    \centering
    \includegraphics[width=\linewidth]{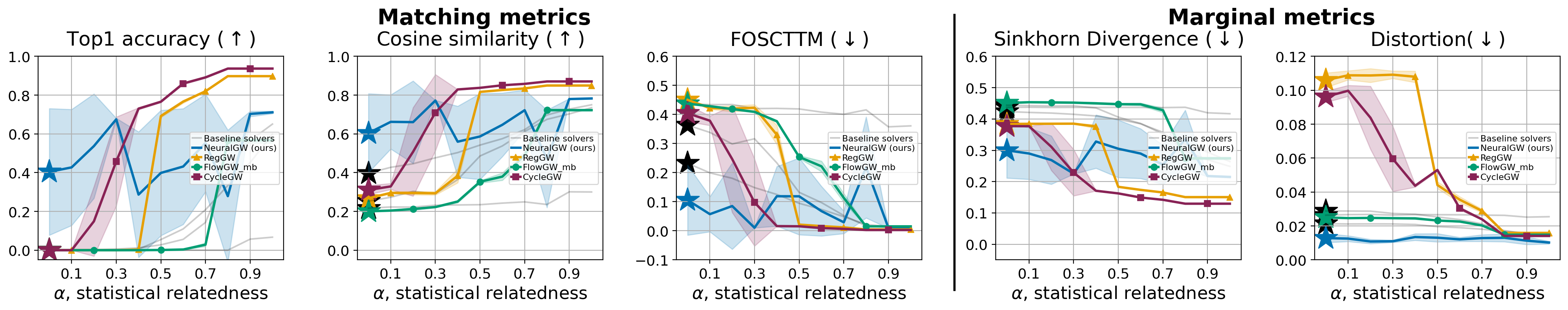} 
    \caption{\small Matching metrics for continuous methods trained on \textbf{statistically related batches} on the Twitter-GloVe dataset ($\text{Source:}100\rightarrow \text{Target:}50$). $N_{train}=400$K, $n_{batch}=500$ and the Top-$k$ accuracy are computed on the \textit{whole} vector space.}
    \label{fig:continuous_paired_twitter_glove_10_50}

\end{figure*}
\textbf{Conclusion.} As mentioned in the main text, it is important to clarify that this setup is both unfair and unrealistic, as it requires providing statistically related batches to the solvers. We observe that this tweak significantly boosts performance compared to the fair setup, which confirms the correct implementation of the solvers and their proper functionality under specific, but artificial, conditions.

\subsection{GloVe}\label{ap:experiments-Glove}

\textbf{Experiments on additional dimensionalities for baseline and continuous solvers}

Figures \ref{fig:discrete_additional} and \ref{fig:continuous_additional} show the performance of the baseline and continuous methods, respectively, across additional dimensional configurations of the GloVe embedding. Specifically, we consider the transformations $50 \rightarrow 100$, $50 \rightarrow 25$, and $25 \rightarrow 50$. Accuracies are reported over the entire \textit{whole} data space of 400K samples. The parameters, training and test samples are the same as in the original experiments. For the continuous solvers we exclude the $25\rightarrow 50$ setup as it produces a result similar to $50\rightarrow 100$ due to the inability of the continuous solvers to transform low-to-high dimensional manifolds.

\begin{figure}[!hbt]
    \begin{subfigure}{1\textwidth}
    \centering
    \includegraphics[width=\linewidth]{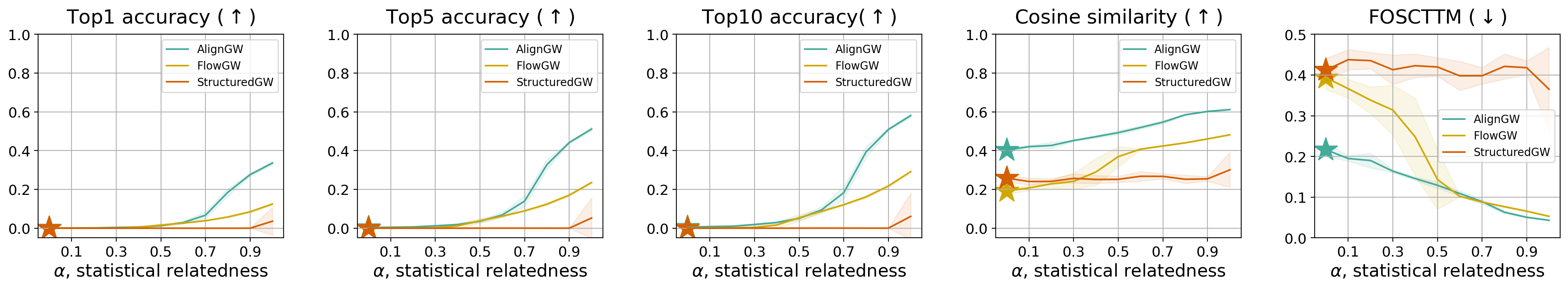} 
    \caption{\small Source: 50 $\rightarrow$ Target:100}
    \label{fig:discrete_twitter_50_100_400K}
    \vspace{-2mm}
    \hspace{2mm}
    \end{subfigure}

   \begin{subfigure}{1\textwidth}
    \centering
    \includegraphics[width=\linewidth]{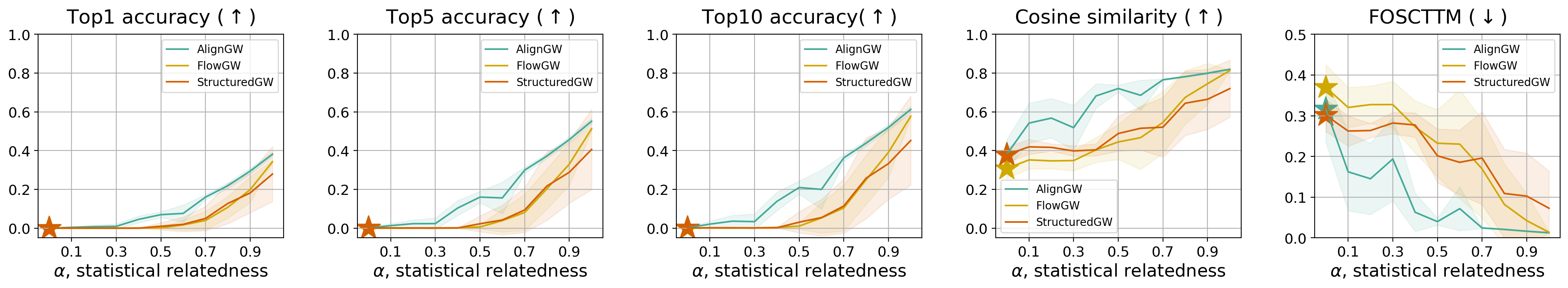} 
    \caption{\small Source: 50 $\rightarrow$ Target:25}
    \label{fig:discrete_twitter_50_25_400K}
    \end{subfigure}

    \vspace{-1mm}
    \begin{subfigure}{1\textwidth}
    \centering
    \includegraphics[width=\linewidth]{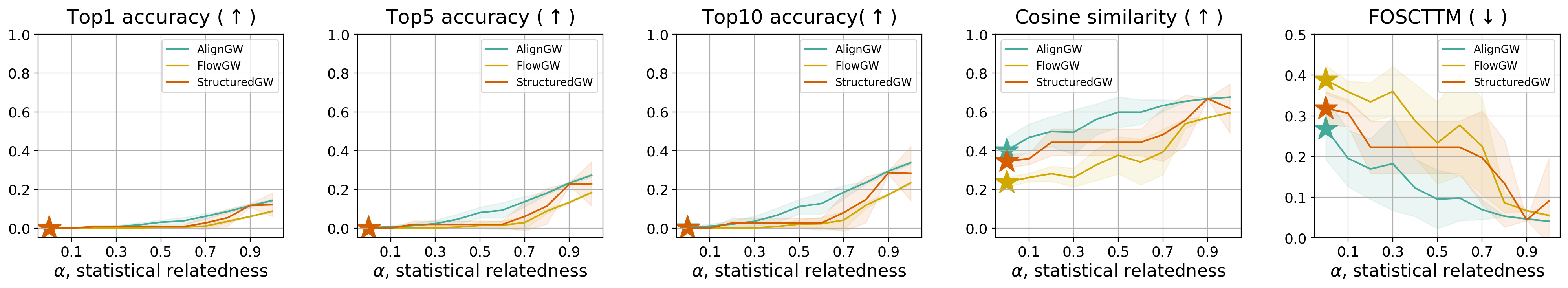} 
    \caption{\small Source: 25 $\rightarrow$ Target:50}
    \label{fig:discrete_twitter_25_50_400K}
    \end{subfigure}

    \vspace{-2mm}
    \caption{\small Performance of baseline GWOT solvers for \textbf{Twitter-GloVe} at different statistical relatedness levels in additional setups. Accuracies computed on the \textit{whole} data space of $400\text{K}$ samples.}
    \label{fig:discrete_additional}
\end{figure}

\vspace{-2mm}
\begin{figure}[!hbt]
    \begin{subfigure}{1\textwidth}
    \centering
    \includegraphics[width=\linewidth]{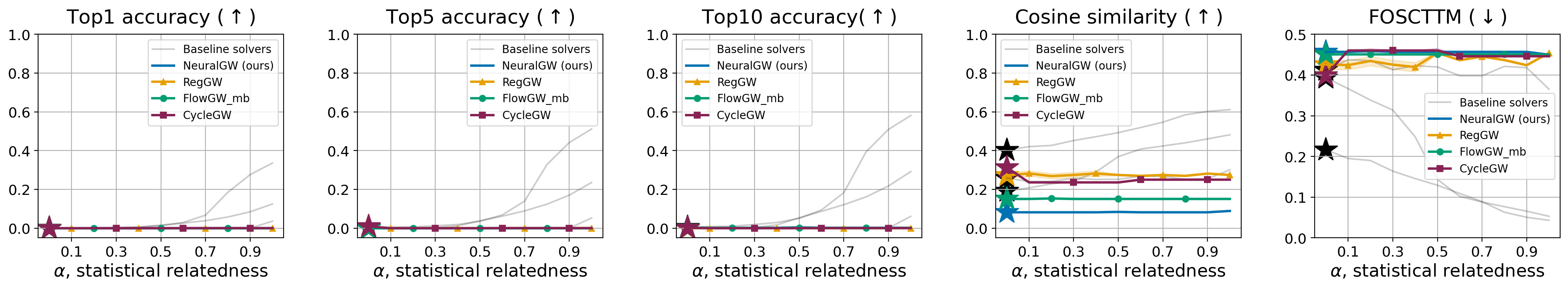} 
    \caption{\small Source: 50 $\rightarrow$ Target:100}
    \label{fig:continuous_twitter_50_100}
    
    \hspace{2mm}
    \end{subfigure}

    \begin{subfigure}{1\textwidth}
    \centering
    \includegraphics[width=\linewidth]{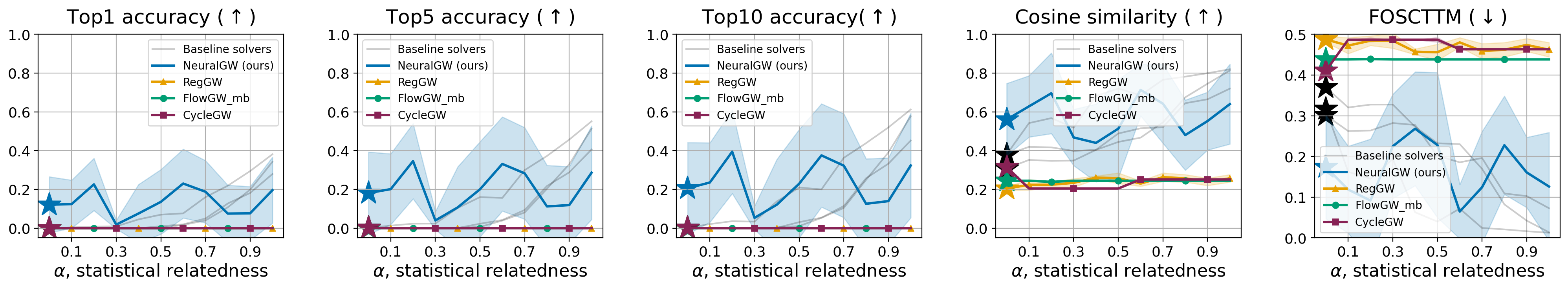} 
    \caption{\small Source: 50 $\rightarrow$ Target:25}
    \label{fig:continuous_twitter_50_25}
    \end{subfigure}
    
    \vspace{-2mm}\caption{\small Performance of the continuous GWOT solvers for the \textbf{Twitter-GloVe} embeddings at different levels of statistical relatedness $\alpha$ in additional dimensional setups.}
    \label{fig:continuous_additional}
\end{figure}

\textbf{Conclusion.} The performance of both baseline and continuous solvers drops significantly when the embedding dimensionality is reduced. This is expected, as lower-dimensional vectors may fail to adequately capture the structural features of the space. Conversely, none of the continuous solvers, including our NeuralGW, are able to effectively handle low-to-high dimensional setups.

\textbf{Additional experiments for NeuralGW for high-to-low dimension in small dataset ($\textbf{N}_{\textbf{train}}=6$K).} Here we consider the same setup used for the baseline methods in \wasyparagraph \ref{subsec:benchmark} to show how a neural approach performs when it is trained using a limited amount of samples, $N_{train}=6$K, see Table \ref{table:neuralgw_small}.

\begin{table}[!hbt]
  \centering
\begin{tabular}{c|c|c|c|c|c|c}
\hline
\multicolumn{1}{l|}{Dimensions}     & \textbf{Stat. Relatedness} & \multicolumn{1}{l|}{\textbf{Top 1}} & \multicolumn{1}{l|}{\textbf{Top 5}} & \multicolumn{1}{l|}{\textbf{Top 10}} & \multicolumn{1}{l|}{\textbf{Cosine similarity}} & \multicolumn{1}{l}{\textbf{FOSCTTM}} \\ \hline
\multirow{3}{*}{100$\rightarrow$50} & \textbf{$\alpha = 0.2$} & 0.000                               & 0.000                               & 0.003                                & 0.100                                           & 0.462                                \\
                                    & \textbf{$\alpha = 0.5$} & 0.000                               & 0.001                               & 0.004                                & 0.107                                           & 0.470                                \\
                                    & \textbf{$\alpha = 1.0$} & 0.000                               & 0.003                               & 0.005                                & 0.117                                           & 0.447                                \\ \hline
\multirow{3}{*}{50$\rightarrow$25}  & \textbf{$\alpha = 0.2$} & 0.000                               & 0.001                               & 0.003                                & 0.165                                           & 0.441                                \\
                                    & \textbf{$\alpha = 0.5$} & 0.001                               & 0.004                               & 0.007                                & 0.176                                           & 0.423                                \\
                                    & \textbf{$\alpha = 1.0$} & 0.000                               & 0.000                               & 0.004                                & 0.174                                           & 0.435                                \\ \hline
\multirow{3}{*}{50$\rightarrow$100} & \textbf{$\alpha = 0.2$} & 0.000                               & 0.002                               & 0.003                                & 0.086                                           & 0.455                                \\
                                    & \textbf{$\alpha = 0.5$} & 0.000                               & 0.003                               & 0.005                                & 0.084                                           & 0.459                                \\
                                    & \textbf{$\alpha = 1.0$} & 0.000                               & 0.002                               & 0.003                                & 0.089                                           & 0.450                                \\ \hline
\multirow{3}{*}{25$\rightarrow$50}  & \textbf{$\alpha = 0.2$} & 0.000                               & 0.002                               & 0.005                                & 0.113                                           & 0.440                                \\
                                    & \textbf{$\alpha = 0.5$} & 0.000                               & 0.000                               & 0.001                                & 0.095                                           & 0.464                                \\
                                    & \textbf{$\alpha = 1.0$} & 0.001                               & 0.003                               & 0.005                                & 0.124                                           & 0.436                                \\ \hline
\end{tabular}
\vspace{0mm}
\caption{\small Performance of the NeuralGW solver for the \textbf{Twitter-GloVe} embeddings at different levels of statistical relatedness $\alpha$ in \textbf{all} setups.  The solver was trained with $N_{train}/2=3000$ samples from a whole space of $400$K, this plot shows results for a testing subset of $2048$ samples, the metrics were computed considering the \textit{whole} space of $400$K samples.}
\label{table:neuralgw_small}

\end{table}

\textbf{Conclusion.}

NeuralGW is unable to model the inner structures when the number of training samples is small (6K). In general, adversarial algorithms like those used to train NeuralGW require plenty of data to obtain meaningful results.

\vspace{-3mm}
\subsection{BPEmb}
\label{ap:muse_experiments}
\vspace{-1mm}

\textbf{MUSE-BPemb, source: English (100) $\rightarrow$ target: Spanish (100)} We considered two different languages for source and target, the dimension of embedding was equal. The total number of samples was $N=60$K, $N_{train}=6$K for baseline solvers and $N_{train}=57$K for NeuralGW, $N_{test}=2048$ for both cases. The metrics were computed using the \textit{whole} data space. See Table \ref{table:muse_en_es_bp} for the results.

\begin{table}[!ht]
\centering
\begin{tabular}{lcccccc}
                                \cline{1-7}       &  \multicolumn{6}{c}{\textbf{Baseline Solvers}}                                                                                                                                                                                                                                                                                                                                                                                                                                                        \\ \cline{1-7} 
\multicolumn{1}{l|}{}                  & \multicolumn{2}{c|}{\textbf{FlowGW}}                                                                                                & \multicolumn{2}{c|}{\textbf{AlignGW}}                                                                                               & \multicolumn{2}{c}{\textbf{StructuredGW}}                                                                                                                                                                                \\ \hline
\multicolumn{1}{l|}{\textbf{$\alpha$}} & \small\textbf{Top 10} & \multicolumn{1}{c|}{\small\textbf{FOSCTTM}} & \small\textbf{Top 10} & \multicolumn{1}{c|}{\small\textbf{FOSCTTM}} & \small\textbf{Top 10}                                                      & \small\textbf{FOSCTTM}                                                      \\ \hline
\multicolumn{1}{l|}{\textbf{0.0}}      & 0.0013                                                & \multicolumn{1}{c|}{0.4222}                                                 & 0.0012                                                & \multicolumn{1}{c|}{0.4056}                                                 & 0.0000                                                                                                     & 0.5037                                                                                                      \\
\multicolumn{1}{l|}{\textbf{0.5}}      & 0.0029                                                & \multicolumn{1}{c|}{0.3983}                                                 & 0.0022                                                & \multicolumn{1}{c|}{0.3768}                                                 & 0.0006                                                                                                     & 0.4541                                                                                                      \\
\multicolumn{1}{l|}{\textbf{1.0}}      & 0.0267                                                & \multicolumn{1}{c|}{0.3321}                                                 & 0.0146                                                & \multicolumn{1}{c|}{0.3085}                                                 & 0.0009                                                                                                     & 0.4443                                                                                                      \\ \hline
                                       & \multicolumn{6}{c}{\textbf{Continuous solvers}}                                                                                                                                                                                                                                                                                                                                                                                                                                                      \\ \cline{1-7} 
\multicolumn{1}{l|}{}                  & \multicolumn{2}{c|}{\textbf{RegGW}}                                                                                                 & \multicolumn{2}{c|}{\textbf{FlowGW (mb)}}                                                                                           & \multicolumn{2}{c}{\textbf{NeuralGW (ours)}}                                                                                                                                                                             \\ \hline
\multicolumn{1}{c|}{\textbf{$\alpha$}} & \small\textbf{Top 10} & \multicolumn{1}{c|}{\small\textbf{FOSCTTM}} & \small\textbf{Top 10} & \multicolumn{1}{c|}{\small\textbf{FOSCTTM}} & \begin{tabular}[c]{@{}c@{}}\small\textbf{Top 10}\\ mean (std)\end{tabular} & \begin{tabular}[c]{@{}c@{}}\small\textbf{FOSCTTM}\\ mean (std)\end{tabular} \\ \hline
\multicolumn{1}{l|}{\textbf{0.0}}      & 0.0001                                                & \multicolumn{1}{c|}{0.4521}                                                 & 0.0006                                                & \multicolumn{1}{c|}{0.4767}                                                 & 0.0518 (0.0633)                                                                                            & 0.3811 (0.1286)                                                                                             \\
\multicolumn{1}{l|}{\textbf{0.5}}      & 0.0009                                                & \multicolumn{1}{c|}{0.4411}                                                 & 0.0007                                                & \multicolumn{1}{c|}{0.4770}                                                 & 0.0284 (0.0564)                                                                                            & 0.4379 (0.1067)                                                                                             \\
\multicolumn{1}{l|}{\textbf{1.0}}      & 0.0013                                                & \multicolumn{1}{c|}{0.4384}                                                 & 0.0008                                                & \multicolumn{1}{c|}{0.4772}                                                 & 0.0748 (0.0751)                                                                                            & 0.3514 (0.1321)                                                                                             \\ \hline
\end{tabular}

\caption{Results for MUSE dataset for English and Spanish as source and target languages, respectively, both are 100-dimensional BP embeddings. }
\label{table:muse_en_es_bp}
\end{table}

\textbf{Conclusion.} We tested a selection of methods for this experiment, but none of them, whether baseline or continuous approaches, were able to produce successful results in this setup. It is worth noting that the MUSE dataset was used in the original AlignGW paper; however, the type of embedding was different, resulting in a higher embedding dimensionality, which may have contributed to better performance in that context.

\vspace{-2mm}
\subsection{Biological dataset}\label{ap:experiments-biology}
\vspace{-2mm}

As seen in the previous experiments, the baselines and NeuralGW solvers have their own limitations and drawbacks heavily linked to their nature. However, in spite of their independent performance, they could partially recover the inner geometry of the domains. In this section, we propose a stress test scenario in which the solvers of the conventional GWOT problem yield poor results.

We explore the performance of the baselines solvers and NeuralGW in a biological dataset called bone marrow \citep{luecken2021a} which is considered in the method that we named FlowGW \citep{klein2023generative}. This dataset consists of 6224 samples of two different RNA profiling methods (ATAC+GEX and ADT+GEX), the samples in each domain belong to the same donors, therefore, the real pairs are known. The dimensionality for source and target domains are 38 and 50, respectively. Results can be found in Table \ref{table:bone_marrow}. All the solvers are tested using 5000 samples for training and the rest for testing.

\begin{table}[!hbt]
  \centering
\begin{tabular}{lcc|cc|cc|cc}
\cline{2-9}
                                    & \multicolumn{2}{c|}{\textbf{FlowGW}}                                     & \multicolumn{2}{c|}{\textbf{AlignGW}}                                    & \multicolumn{2}{c|}{\textbf{StructuredGW}}                               & \multicolumn{2}{c}{\textbf{NeuralGW}}                                    \\ \hline
\multicolumn{1}{l|}{\textbf{$\alpha$}}       & \small\textbf{Top 10} & \small\textbf{FOSCTTM} & \small\textbf{Top 10} & \small\textbf{FOSCTTM} & \small\textbf{Top 10} & \small\textbf{FOSCTTM} & \small\textbf{Top 10} & \small\textbf{FOSCTTM} \\ \hline
\multicolumn{1}{l|}{\textbf{\textbf{0.2}}} & 0.004           & 0.486                                                  & 0.004           & 0.488                                                  & 0.001           & 0.489                                                  & 0.003           & 0.479                                                  \\
\multicolumn{1}{l|}{\textbf{\textbf{0.5}}} & 0.004           & 0.489                                                  & 0.004           & 0.483                                                  & 0.004           & 0.487                                                  & 0.003           & 0.514                                                  \\
\multicolumn{1}{l|}{\textbf{\textbf{1.0}}} & 0.004           & 0.49                                                   & 0.004           & 0.484                                                  & 0.004           & 0.49                                                   & 0.003           & 0.459                                                  \\ \hline
\end{tabular}

  \caption{Results for bone marrow dataset.}
  \label{table:bone_marrow}
  \vspace{-3mm}
\end{table}
\vspace{-1mm}

\textbf{Conclusion.} In all the cases, the solvers could not properly replicate the inner geometry of the distributions even for totally statistically related setups, this leaded to get metrics corresponding to random guessing, i.e. accuracies close to 0 and FOSCTTM close to 0.5. There is one case of success for a solver trained on this dataset which corresponds to FlowGW \citep{klein2023generative}, however, the reported results in their paper were obtained using a fused-GW solver instead of a conventional GW.

Finally, we can state that there is no general solver for the GWOT problem, all the currently available methods struggle when dealing with real world scenarios, i.e. statistically unrelated data, or with real world datasets, i.e. not consistent inner structures.  

\vspace{-4mm}
\section{Solvers' implementation details}
\label{ap:code}

All the experiments were done without any normalization for the source and target vectors and for all the studied methods (baselines and NeuralGW). A total of ten repetitions were performed. It is important to clarify that the parameters listed below are the ones we used to align the embeddings, they may require some tweaks to make them work in the toy setup.

\textbf{StructuredGW.\citep{sebbouh2024structured}} We used the code from the official repository:
\vspace{-1mm}
\begin{center}
\url{https://github.com/othmanesebbouh/prox_rot_aistats}
\end{center}

As specified in Section \wasyparagraph \ref{sec:related_work}, the algorithm uses an iterative solver that updates the cost matrix $\textbf{T}$ by implementing several methods depending on the type of regularization, we only use the exact computation without any regularization. The plan $\pi$ is also updated every iteration by performing Sinkhorn iterations, we set this number of iterations to 1000. The entropy is set to $\varepsilon=1e$-4. The total number of iterations is set to 200 or until convergence.

In this implementation, the authors use the Optimal Transport Tools (OTT) library \citep{cuturi2022optimal}. The computation time per repetition until convergence was 30 minutes in average on a CPU.

\textbf{AlignGW.\citep{alvarez2018gromov}} We use the official implementation of the method taken from the repository:
\vspace{-1mm}
\begin{center}
    \url{https://github.com/dmelis/otalign}
\end{center}

We set the entropy term to $\varepsilon=$1e-3 and use the cosine distance to compute the source and target intra-cost matrices $\textbf{C}^x$ and $\textbf{C}^y$. We later normalize them by dividing them by their respective means as proposed in the original implementation. The model was trained on a CPU and the average training time was 10 minutes per repeat. The implementation uses the POT library. To get the predicted target from the obtained couplig $\mathbf{T}$, we first compute the indices \texttt{idx=np.argmax(T, axis=1)} and then the chosen target samples are \texttt{y\_pred=y[idx]}. We chose this strategy as it gave better performance. Then, we train a \texttt{scikit-learn}'s MLPRegressor on top of this prediction as an inference method for test data, the parameters are: hidden\_layer\_sizes=256, random\_state=1, max\_iter=500. 

\textbf{FlowGW. \citep{klein2023generative}} We used the implementation provided in the OTT library for the GENOT with slight modifications to adapt it to our pipeline. The hyperparameters for the vector field were as follows: Number of frequencies: 128, layers per block: 8, hidden dimension: 1024, activation function: SiLu, optimizer: AdamW (lr=1e-4). The Gromov-Wasserstein solver was set to work with entropy $\varepsilon=1$e-3 and using cosine distance to compute the intra-domain matrices.

\textbf{RegGW. \citep{uscidda2024disentangledrepresentationlearninggeometry}} For the sake of fairness, our implementation of this solver is based on the publicly available implementation for the Monge gap regularizer from the OTT library \citep{cuturi2022optimal}. To compute the Gromov-Wasserstein distance we used the GW solver from the library and took the entropy regularized cost. The following parameters were used for training: $\varepsilon_{fit}=0.01$, $\varepsilon_{reg}=0.001$, $\lambda = 1$. The transport model was parametrized as an MLP with $[512, 256, 256]$ dimensions for the hidden layers, the optimizer learning rate was $1$e-$4$ and a batch size of $256$.

\textbf{CycleGW \citep{Zhang2021CycleCP}}. We used the code provided by the authors in their repository:
\vspace{-1mm}
\begin{center}
    \url{https://github.com/ZhengxinZh/GMMD}
\end{center}

As in the original implementation, we used fully connected neural networks (FCNN) to parametrize $f$ and $g$, in both cases the network consisted on a single layer with $512$ neurons, and trained using the Adam optimizer with a learning rate of $1e$-3 (as suggested in the original paper). Both regularization parameters, $\lambda_x$ and $\lambda_y$ in the original paper, were set to $0.1$. The multiplier of the distortion term was set to $5$e-4. In spite of following the original implementation, it was not possible to make the solver work for our setups beyond the toy experiment.

\vspace{-3mm}
\subsection{NeuralGW.}
\label{ap:neuralgw_alg}

The innerGW problem in \eqref{eq:innerGW} can be optimized using our Algorithm \ref{alg:alg_1}.

\begin{algorithm}
\caption{Training algorithm for Neural Gromov-Wasserstein OT}\label{alg:alg_1}

\textbf{Input:}Distributions $\mathbb{P}$ and $\mathbb{Q}$ obtained from samples.

\textbf{Output:}Optimal matrix $P_{\omega}$, critic $f_{\theta}$ and transport map $T_{\gamma}$.

\For{$i=1, 2, 3, \dots, n_{epochs}$}
{
    
    \For{$i=1, 2, 3, \dots, k_P$}
        { 
            Sample batch from source and target distributions $X\sim\mathbb{P}$, $Y\sim\mathbb{Q}$.
            
            Compute $P$ loss $\mathcal{L}_P=-\frac{1}{N}\sum_{n=1}^N\langle \textbf{P}_{\omega}\textbf{x},T_{\gamma}(\textbf{x})\rangle$
            
            Gradient step over $\omega$ using $\frac{\partial\mathcal{L}_P}{\partial \omega}$

    \For{$j=1, 2, 3, \dots, k_f$}
    {

        \For{$k=1, 2, 3, \dots, k_T$}
        {
        Sample batch from source and target distributions $X\sim\mathbb{P}$, $Y\sim\mathbb{Q}$.
        
        Compute mover loss $\mathcal{L}_T=-\frac{1}{N}\sum_{n=1}^N\langle \textbf{P}_{\omega}\textbf{x},T_{\gamma}(\textbf{x})\rangle-\frac{1}{N}\sum_{n=1}^Nf_{\theta}(T(\textbf{x}))$
        
        Gradient step over $\gamma$ using $\frac{\partial\mathcal{L}_T}{\partial \gamma}$

        }
        Sample batch from source and target distributions $X\sim\mathbb{P}$, $Y\sim\mathbb{Q}$.
        
        Compute critic loss $\mathcal{L}_f=\frac{1}{N}\sum_{n=1}^Nf_{\theta}(T_{\gamma}(\textbf{x}))-\frac{1}{N}\sum_{n=1}^Nf_{\theta}(\textbf{y})$
        
        Gradient step over $\theta$ using $\frac{\partial\mathcal{L}_f}{\partial \theta}$

    }
    }
}
\end{algorithm}

Every experiment runs for 150 epochs. Each network iteration takes new i.i.d. samples from source and target datasets. $f$ and $T$ are parametrized using multi-layer perceptrons with $n_l$ with width $h$. Matrix $\textbf{P}$ is parameterized by a single-layer linear neural network, we use GeoTorch \cite{lezcano2019trivializations} to constrain each row of its rows to lie on the unit sphere, ensuring row-wise normalization is maintained throughout training. These models are trained for $k_f,k_T$ and $k_p$ iterations, respectively. These parameters were chosen following the inner structure of the min-max-min problem, we optimized the inner transport model $T_{\gamma}$ more often than the others, from here we fine-tuned the values to get good metrics. The implementation details can be seen in Table \ref{table:neuralgw_details}. The batch size is 512 for the experiments with 400K, 90K and 600K samples and 64 for the experiments with 6K. Every run takes around to 2 hours running (for large-scale experiments) on a GPU NVIDIA A100. Our code is written in \texttt{PyTorch}.

\begin{table}[!hbt]
  \caption{Parameters for NeuralGW.}
  \label{table:neuralgw_details}
  \centering
\begin{tabular}{l|c|c|c|c|c}
\hline
                                                                                                                    & Model & $k$ & $n_l$              & $h$                  & lr                    \\ \hline
\multirow{3}{*}{\begin{tabular}[c]{@{}l@{}}100$\rightarrow$50\\ 50$\rightarrow$25\\ 25$\rightarrow$50\end{tabular}} & $f$   & $k_f=$1   & \multirow{3}{*}{4} & \multirow{3}{*}{1024} & \multirow{3}{*}{1e-4} \\ \cline{2-3}
                                                                                                                    & $T$   & $k_T=$5  &                    &                      &                       \\ \cline{2-3}
                                                                                                                    & $P$   & $k_P=$1   &                    &                      &                       \\ \hline
\multirow{3}{*}{50$\rightarrow$25}                                                                                  & $f$   & $k_f=$1   & \multirow{3}{*}{4} & \multirow{3}{*}{256} & \multirow{3}{*}{1e-4} \\ \cline{2-3}
                                                                                                                    & $T$   & $k_T=$10  &                    &                      &                       \\ \cline{2-3}
                                                                                                                    & $P$   & $k_P=$1   &                    &                      &                       \\ \hline
\end{tabular}
\end{table}

To stabilize NeuralGW during training on word and voice embeddings, we incorporate the well-known $R_1$ gradient penalty \cite{mescheder2018trainingmethodsgansactually} into the critic loss $\mathcal{L}_f$, following GAN-based training methods. The resulting objective, minimized with respect to the critic parameters $\theta$, is 

$$
\mathcal{L}_f =
\frac{1}{N}\sum_{n=1}^N f_{\theta}(T_{\gamma}(\mathbf{x}_n))
-\frac{1}{N}\sum_{n=1}^N f_{\theta}(\mathbf{y}_n)
+\lambda\frac{1}{N}\sum_{n=1}^N
\left\|\nabla_{\mathbf{y}_n}f_{\theta}(\mathbf{y}_n)\right\|_2^2,
$$

where $\lambda=0.1$ controls the regularization strength.
with $\lambda=0.1$. All the results for NeuralGW on embeddings use this regularizer for training the critic.

\vspace{-3mm}

\section{Broader discussions}
\vspace{-3mm}
\subsection{Discrete GW solvers under low statistical relatedness data scenario}\label{app-discrete-gw-issue}
\vspace{-2mm}

\begin{wrapfigure}{r}{0.6\textwidth}
\centering\includegraphics[width=1.0\linewidth]{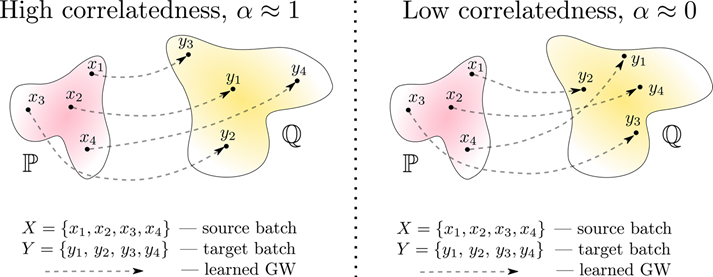} 
\caption{\centering\small Discrete GW maps fitted under high (left) and low (right) levels of statistical relatedness.}
\label{fig:discrete_gw_issue}
\vspace{-4mm}
\end{wrapfigure}

Our experimental results (Section \ref{subsec:benchmark} of the main text) testify that the discrete baseline solvers perform unsatisfactory when the level of statistical relatedness $\alpha$ in the data tends to zero. We hypothesis that the main reason behind this behavior is as follows. Small amount of data used for discrete solvers hardly could ``catch'' the intrinsic geometry of the underlining distribution. When we apply discrete GW solver, the Gromov-Wasserstein mapping is learned between the geometries induced by sample distributions, not original distributions, see Figure \ref{fig:discrete_gw_issue} for illustration. These ``induced'' geometries may be different from the original ones, they may have other symmetries and other properties. Matching them may result in Gromov-Wasserstein map which is quite different from the real GW map. Importantly, our hypothesis above is (indirectly) supported by recently derived sample complexity rates for Gromov-Wasserstein distance \citep{zhang2024gromov}. The authors consider GWOT problem \eqref{eq:gwot-problem} for quadratic costs $c_{\mathcal{X}}, c_{\mathcal{Y}}$. In this case, they observe that the sample convergence rate of the GW distance between empirical (sample-based) measures to the GW distance between reference measures ($\bbP, \bbQ$) is of the order $N^{\frac{-2}{\min{d_{\X}, d_\Y}}}$, where $N$ is the number of samples. Note that this discouraging exponential rate corresponds exactly to statistically unrelated ($\alpha \approx 0$) setup, where the discrete samples from the source and target reference distributions are independent. Besides, the sample complexity of \textit{GW maps}, not distances, may be even worse.

On the other hand, when the level of statistical relatedness is high ($\alpha = 1$), the GW problem is reduced to finding the proper permutation of the data, see Figure \ref{fig:discrete_gw_issue}, left part. The true solution of discrete GW in this case \textit{coincides} with the true underlining GW map. If the learned map properly generalizes to new (unseen) samples, then the resulting performance is expected to be satisfactory. It is the behaviour we observe in our experiments, Section \ref{subsec:benchmark}.

\section{NeuralGW initialization strategy}
\label{ap:neuralgw_init}

To alleviate the stability challenges arising from the min‑max‑min optimization in NeuralGW, we initialize the learnable matrix $\textbf{P}_{\omega}$ using the solution of the InnerGW problem (Equation~\eqref{eq:innerGW}) between two non‑centered Gaussian distributions. We denote this optimal matrix by $\mathbf{P^*_{g}}$, which we obtain by combining \citep[Lemma 4.4.2]{vayer2020contribution} with results from \citep[Section 3]{delon2022gromov}. Our initialization procedure consists of the following steps:

\begin{enumerate}[leftmargin=*]
    \item Given the empirical source and target distributions $X = \{ \x_i\}_{i = 1}^{N}$ and $Y = \{ \y_i \}_{i = 1}^{N}$, we compute $m_x = \frac{1}{N}\sum_{i=1}^{N}\mathbf{x}_i$, $m_y = \frac{1}{N}\sum_{i=1}^{N}\mathbf{y}_i$, $\Sigma_x = \operatorname{Var}(X)$, and $\Sigma_y = \operatorname{Var}(Y)$.
    \item We determine the specific optimal $\mathbf{P^*_{g}}$ that corresponds to these empirical means and covariances.
    \item In Algorithm~\ref{alg:alg_1}, we freeze the training of $\mathbf{P}_{\omega}$ and train only $T_{\gamma}$ and $f_{\theta}$ for 15 epochs.
    \item After this initial pre‑training of $T_{\gamma}$ and $f_{\theta}$, we unfreeze $\mathbf{P}_{\omega}$ and resume regular training as described in Algorithm~\ref{alg:alg_1}.
\end{enumerate}

This strategy improves the training stability of NeuralGW. For the Twitter‑GloVe and MUSE‑BPEmb datasets, it allows us to remove the $R_1$ regularizer presented in Appendix~\ref{ap:neuralgw_alg}. However, the $R_1$ regularizer remains necessary for the VoxCeleb2‑WNTA dataset.

\end{document}